\documentclass[sigconf]{acmart}
\AtBeginDocument{%
  }

\copyrightyear{2026}
\acmYear{2026}
\setcopyright{cc}
\setcctype{by}
\acmConference[KDD '26]{Proceedings of the 32nd ACM SIGKDD Conference on Knowledge Discovery and Data Mining V.2}{August 09--13, 2026}{Jeju Island, Republic of Korea}
\acmBooktitle{Proceedings of the 32nd ACM SIGKDD Conference on Knowledge Discovery and Data Mining V.2 (KDD '26), August 09--13, 2026, Jeju Island, Republic of Korea}
\acmDOI{10.1145/3770855.3818013}
\acmISBN{979-8-4007-2259-2/2026/08}

\usepackage{amsmath}

\usepackage{amssymb}
\usepackage{mathtools}
\usepackage{amsthm}

\usepackage{xspace}
\usepackage{multirow}
\usepackage{xurl}

\usepackage{algorithm}
\usepackage{algorithmic}
\usepackage{wrapfig}
\usepackage{caption}
\usepackage{subcaption}
\usepackage[table]{xcolor}

\usepackage{pifont}

\usepackage{enumitem}

\theoremstyle{plain}
\newtheorem{theorem}{Theorem}

\newtheorem{corollary}{Corollary}
\newtheorem*{thm:flexours}{Theorem 1}

\newcommand{\RR}{\mathbb{R}}

\newcommand{\set}[1]{\mathcal{#1}}

\providecommand{\sE}{\ensuremath{\set{E}}}

\providecommand{\sN}{\ensuremath{\set{N}}}
\providecommand{\sO}{\ensuremath{\set{O}}}

\providecommand{\sS}{\ensuremath{\set{S}}}

\providecommand{\sV}{\ensuremath{\set{V}}}
\providecommand{\sX}{\ensuremath{\set{X}}}

\providecommand{\sZ}{\ensuremath{\set{Z}}}

\providecommand{\shE}{\ensuremath{\widehat{\set{E}}}}

\providecommand{\shV}{\ensuremath{\widehat{\set{V}}}}

\renewcommand{\vec}[1]{{\bf{#1}}}
\providecommand{\va}{\ensuremath{\vec{a}}}

\providecommand{\ve}{\ensuremath{\vec{e}}}

\providecommand{\vk}{\ensuremath{\vec{k}}}

\providecommand{\vx}{\ensuremath{\vec{x}}}

\providecommand{\vq}{\ensuremath{\vec{q}}}

\renewcommand{\vv}{\ensuremath{\vec{v}}}

\providecommand{\vha}{\ensuremath{\widehat{\vec{a}}}}

\newcommand{\ours}[0]{\textup{MAVN}\xspace}
\newcommand{\gcn}[0]{\textup{GCN}\xspace}
\newcommand{\gine}[0]{\textup{GINE}\xspace}
\newcommand{\gatedgcn}[0]{\textup{GatedGCN}\xspace}

\newcommand{\sage}[0]{\textup{GraphSAGE}\xspace}

\newcommand{\gat}[0]{\textup{GAT}\xspace}

\newcommand{\drew}[0]{\textup{DRew}\xspace}
\newcommand{\laser}[0]{\textup{LASER}\xspace}

\newcommand{\stgcn}[0]{\textup{S$^{2}$GCN}\xspace}

\newcommand{\amp}[0]{\textup{AMP}\xspace}
\newcommand{\cognn}[0]{\textup{Co-GNN}\xspace}
\newcommand{\nba}[0]{\textup{NBA}\xspace}

\newcommand{\gps}[0]{\textup{GraphGPS}\xspace}
\newcommand{\exphormer}[0]{\textup{Exphormer}\xspace}
\newcommand{\grit}[0]{\textup{GRIT}\xspace}
\newcommand{\gvit}[0]{\textup{Graph ViT}\xspace}

\newcommand{\polynormer}[0]{\textup{Polynormer}\xspace}
\newcommand{\polyformer}[0]{\textup{PolyFormer}\xspace}

\newcommand{\megraph}[0]{\textup{MeGraph}\xspace}
\newcommand{\prmpnn}[0]{\textup{PR-MPNN}\xspace}
\newcommand{\iprmpnn}[0]{\textup{IPR-MPNN}\xspace}
\newcommand{\geaet}[0]{\textup{GEAET}\xspace}

\newcommand{\nsquare}[0]{\textup{N$^{2}$}\xspace}

\newcommand{\uniconv}[0]{\textup{UniGCN}\xspace}
\newcommand{\revgnn}[0]{\textup{ReP}\xspace}
\newcommand{\jdr}[0]{\textup{JDR}\xspace}
\newcommand{\comfy}[0]{\textup{ComFy}\xspace}

\newcommand{\pepfunc}[0]{\textsl{Peptides-func}\xspace}
\newcommand{\pepstruct}[0]{\textsl{Peptides-struct}\xspace}
\newcommand{\pascal}[0]{\textsl{PascalVOC-SP}\xspace}
\newcommand{\coco}[0]{\textsl{COCO-SP}\xspace}

\newcommand{\empire}[0]{\textsl{roman-empire}\xspace}
\newcommand{\amazon}[0]{\textsl{amazon-ratings}\xspace}
\newcommand{\minesweeper}[0]{\textsl{minesweeper}\xspace}
\newcommand{\tolokers}[0]{\textsl{tolokers}\xspace}
\newcommand{\questions}[0]{\textsl{questions}\xspace}
\newcommand{\synunder}[0]{\textsl{Tree-LeafCount}\xspace}
\newcommand{\synsquash}[0]{\textsl{Tree-NeighborsMatch}\xspace}
\newcommand{\treelc}[0]{\textsl{Tree-LC}\xspace}
\newcommand{\treenb}[0]{\textsl{Tree-NM}\xspace}
\newcommand*{\boldone}{\text{\usefont{U}{bbold}{m}{n}1}}

\definecolor{c1}{RGB}{25, 100, 175}
\definecolor{c2}{RGB}{30, 140, 0}
\definecolor{c3}{RGB}{225, 100, 0}

\definecolor{c4}{RGB}{156, 97, 49}
\definecolor{c5}{RGB}{160, 32, 240}

\begin{document}

%%
%% The "title" command has an optional parameter,
%% allowing the author to define a "short title" to be used in page headers.

\title{Learn When and Where to Connect: Adaptive Virtual Nodes for Dynamic Message Passing on Graphs}

%%
%% The "author" command and its associated commands are used to define
%% the authors and their affiliations.
%% Of note is the shared affiliation of the first two authors, and the
%% "authornote" and "authornotemark" commands
%% used to denote shared contribution to the research.
%\author{}
%\email{}
%\orcid{}
%\author{}
%\authornotemark[1]
%\email{}
%\affiliation{%
%  \institution{}
%  \city{}
%  \state{}
%  \country{}
%}
\author{Jaejun Lee}
\affiliation{%
\department{School of Computing}
  \institution{KAIST}
  \city{Daejeon}
  \country{Republic of Korea}
}
\email{jjlee98@kaist.ac.kr}

\author{Joyce Jiyoung Whang}
\authornote{Corresponding author.}
\affiliation{%
  \department{Department of AI Computing}
  \institution{KAIST}
  \city{Daejeon}
  \country{Republic of Korea}
}
\email{jjwhang@kaist.ac.kr}
%%
%% By default, the full list of authors will be used in the page
%% headers. Often, this list is too long, and will overlap
%% other information printed in the page headers. This command allows
%% the author to define a more concise list
%% of authors' names for this purpose.
%\renewcommand{\shortauthors}{}

%%
%% The abstract is a short summary of the work to be presented in the
%% article.
\begin{abstract}
While Virtual Nodes (VNs) are often utilized in Message Passing Neural Networks (MPNNs) to facilitate effective message passing, existing VN-based methods have limitations, such as constraining all nodes to connect to the same number of VNs, fixing the connections before applying MPNNs, and connecting a node to a VN independently of the other nodes that connect to the same VN. We propose \ours, an end-to-end differentiable MPNN framework that allows non-constrained connections between nodes and VNs and dynamically introduces VNs on demand in response to evolving node representations across layers. Specifically, \ours learns to adaptively determine \textit{when} (at which layer) and \textit{where} (to which nodes) to introduce and connect VNs based on the relative importance of connections. From a pool of candidate VNs, \ours selects the necessary VNs in each layer, where each selected VN is connected to a nonempty subset of nodes, guided by a dual-perspective scoring mechanism that jointly captures the nodes' preferences for VNs and the VNs' preferences for nodes. We theoretically prove that for any node-VN connectivity pattern, there exists a set of \ours's parameters that can simulate the pattern. Experiments on nine real-world datasets demonstrate that \ours consistently improves the performance of backbone MPNNs, achieving up to 46.5\% improvement over the backbones and outperforms the baselines. 
\end{abstract}

%%
%% The code below is generated by the tool at http://dl.acm.org/ccs.cfm.
%% Please copy and paste the code instead of the example below.
%%


%%
%% Keywords. The author(s) should pick words that accurately describe
%% the work being presented. Separate the keywords with commas.
%% A "teaser" image appears between the author and affiliation
%% information and the body of the document, and typically spans the
%% page.

%%
%% This command processes the author and affiliation and title
%% information and builds the first part of the formatted document.
\maketitle
\newcommand\kddavailabilityurl{https://doi.org/10.5281/zenodo.20446608}
\newcommand\githuburl{https://github.com/bdi-lab/MAVN}
\ifdefempty{\kddavailabilityurl}{}{
\begingroup\small\noindent\raggedright\textbf{Resource Availability:}\\
% please change the following context to include multiple artifacts if necessary, including data, models, code, etc.
The code has been made publicly available at \url{\kddavailabilityurl}. The latest version is at \url{\githuburl}.
\endgroup
}

\section{Introduction}

Graph Neural Networks (GNNs)~\cite{gnnjour, gnnconf, gnnsurv} are a class of neural networks designed for graph data, and many GNN models are categorized as Message Passing Neural Networks (MPNNs)~\cite{mpnn}, which leverage the graph structure to iteratively compute node representations by aggregating information from neighboring nodes. While effective, MPNNs exhibit several well-known limitations. For example, insufficiently deep layers cause MPNNs to fail to capture long-range dependencies appropriately (\textit{under-reaching})~\cite{undr}, while excessively deep layers can lead to overly similar receptive fields, making representations indistinguishable (\textit{over-smoothing})~\cite{ovsm}. Also, fixed-size vectors in MPNNs may struggle to encode information from extensive receptive fields (\textit{over-squashing})~\cite{ovsq}.

\begin{figure}
\centering
\includegraphics[width=0.92\linewidth]{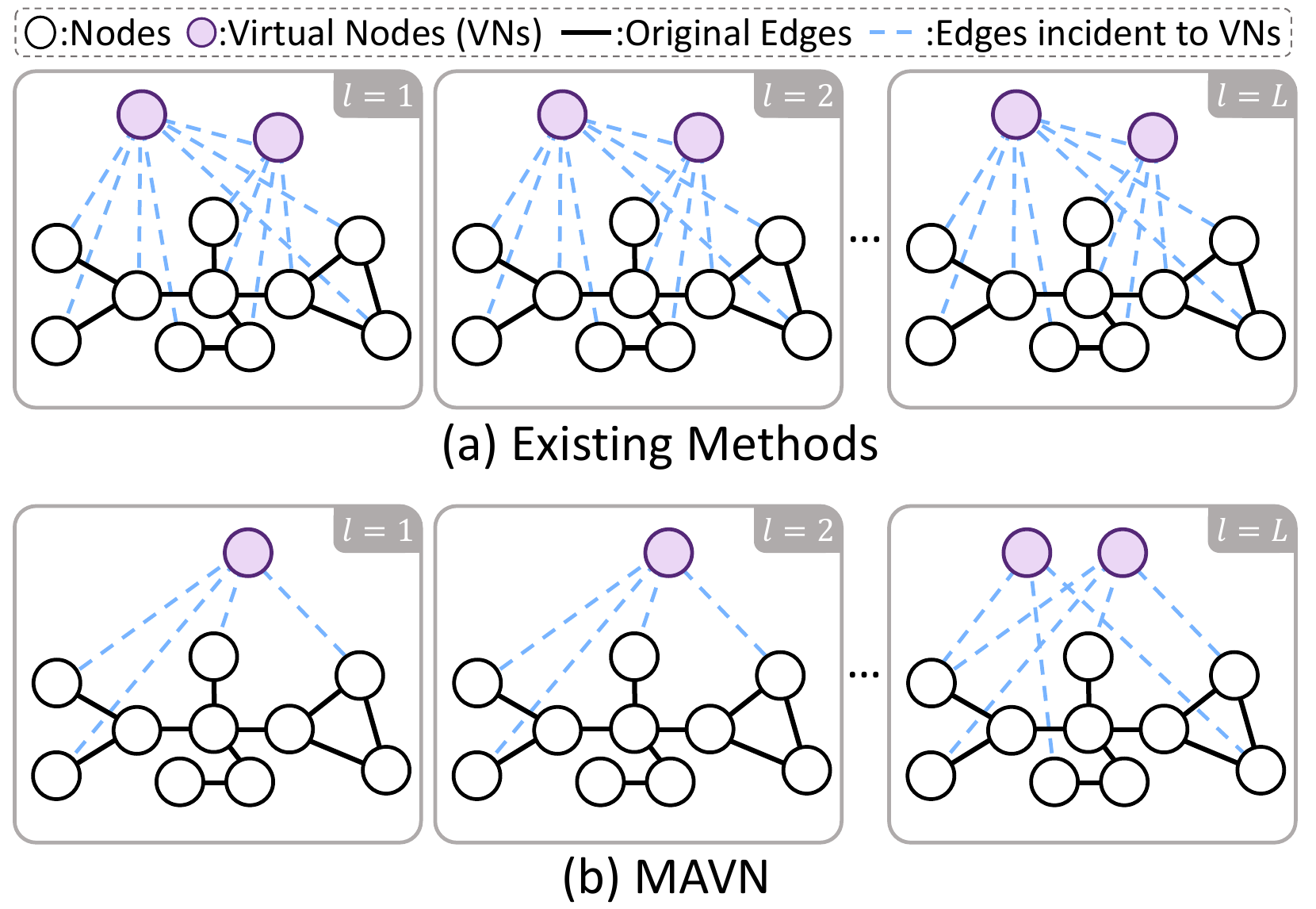}
\caption{Comparison of existing VN-based methods and \ours. While existing methods connect all nodes to the same number of VNs and fix the connections across all layers, \ours allows nodes to connect to a variable number of VNs and enables VNs to emerge at different layers.}
\label{fig:cmp}
\end{figure}

While graph rewiring methods~\cite{jdr,prmpnn} add or remove edges to address the issues, they usually incur quadratic complexity in the number of nodes, as every pair of nodes becomes a candidate for consideration. A more scalable solution is introducing Virtual Nodes (VNs), auxiliary nodes connected to the original nodes~\cite{mpnn} with linear complexity. For example, \nsquare~\cite{nsquare} connects all nodes to all VNs and computes node-VN edge weights at each layer, while \iprmpnn~\cite{iprmpnn} first computes node representations via an upstream MPNN, then links each node to a fixed number of VNs before applying a downstream MPNN. Figure~\ref{fig:cmp}(a) illustrates the graph structure produced by existing VN-based approaches, where each node connects to the same number of VNs; these connections are predetermined before applying an MPNN and fixed across all layers. When establishing a connection between a node and a VN, these methods disregard other nodes that connect to the same VN.

However, this constrained design fails to consider that individual nodes may require varying numbers of connections and that additional VNs may need to be introduced at specific layers. For example, while hub nodes are likely to effectively communicate with other nodes, low-degree nodes can benefit from auxiliary connections. Furthermore, while some nodes may require VNs across all layers, others might benefit from VNs only in some layers. Crucially, since nodes connected to the same VN exchange messages at the subsequent layers, establishing a node-VN connection should account for the other nodes connecting to that VN. Consequently, it is desirable that nodes connect to varying numbers of VNs, VNs should be adaptively introduced at different layers, and node-VN connections should be established jointly rather than in isolation.

We propose an end-to-end differentiable MPNN framework that allows unconstrained and non-uniform connections between nodes and VNs and enables adding VNs at any layer of MPNN when needed by inspecting the need for new VNs using the updated representations at each layer. We named our framework \ours (dynamic \underline{M}essage passing with \underline{A}daptive \underline{V}irtual \underline{N}odes), pronounced Maven. Figure~\ref{fig:cmp}(b) briefly visualizes how \ours differs from existing methods. \ours learns to determine \textit{when} (i.e., at which layer) and \textit{where} (i.e., to which nodes) to introduce and connect VNs by accounting for the relative importance of candidate connections derived from the representations of nodes and VNs. At each layer, \ours identifies and selects the necessary VNs, and if present, connects them to nodes by considering both the nodes' preference for VNs and the VNs' preference for nodes. During this process, only node representations are utilized, enabling \ours to be applied to any MPNN architecture. \ours facilitates dynamic message passing, where message passing paths evolve across layers by emerging VNs, which are adaptively introduced and connected to nodes based on the representations returned by each layer of MPNN. Furthermore, we theoretically prove that for any connectivity pattern between nodes and VNs, the parameters of \ours can be configured to construct such a pattern, demonstrating its capability to generate any required message passing paths involving VNs. Our key contributions are summarized as follows:

\begin{itemize}
\item We propose \ours, a framework that introduces VNs and connects them to nodes at each layer in an adaptive and unconstrained manner. This design enables \ours to generate message passing paths tailored to each graph, mitigating over-squashing and under-reaching.
\item In \ours, employing VNs and updating node representations are interleaved, enabling effective and dynamic message propagation. Also, since \ours utilizes only node representations, it can be applied to any MPNN.
\item We theoretically prove that for any graph structure augmented by connecting VNs to nodes, there exists a single-layer \ours that constructs the same structure. Consequently, \ours can simulate graph structures of existing VN-based methods (e.g., \nsquare~\cite{nsquare} and \iprmpnn~\cite{iprmpnn}).
\item \ours outperforms 28 state-of-the-art methods on nine datasets for node and graph level tasks, consistently improving the backbone MPNNs' performance.
\end{itemize}

\section{Related Work}

\paragraph{Graph Rewiring}
Graph rewiring methods~\cite{prmpnn, laser, borf} modify the connectivity of a graph to handle over-squashing or over-smoothing. Several approaches leverage structural properties such as the spectral gap~\cite{fosr, diffwire} or the effective resistance~\cite{ovsqer}. While \drew~\cite{drew} proposes a layer-dependent rewiring, some other methods~\cite{comfy, jdr} utilize node features combined with the spectral characteristics of the graph. However, these graph rewiring methods can only locally modify direct connections between a node pair and cannot model interactions between a broader subset of nodes.

\paragraph{VN-based Methods}
While the concept of a VN was originally introduced to provide global information to all nodes~\cite{vngt, exphormer}, VNs have since been used in various ways~\cite{megraph, geaet, vcr}. For example, \iprmpnn~\cite{iprmpnn} learns a probability distribution over node-VN edges and samples $k$ edges per node. \nsquare~\cite{nsquare} connects each VN to all nodes with edge weights varying every layer. Both \iprmpnn and \nsquare require all nodes to connect to the same number of VNs, which are fixed at all layers, and enforce that all VNs are fully connected. Importantly, these methods establish each node-VN connection independently of one another, overlooking that the nodes sharing a VN interact via message passing. In contrast, \ours allows nodes to adaptively determine which VNs to connect, including none, by accounting for the other nodes in the graph, while also enabling VNs to determine which VNs to connect to at each layer.

\paragraph{Graph Structural Learning}
Graph Structure Learning (GSL) methods~\cite{gslsurv, opengsl} jointly optimize the graph topology and node representations by learning to generate a refined adjacency matrix, aiming to mitigate noisy connections or construct task-specific graph structures. For instance, Pro-GNN~\cite{prognn} treats GSL as a defense mechanism to recover a clean graph from a graph perturbed by adversarial attacks, while OAGS~\cite{oags} utilizes both node features and labels to guide GSL for node classification. Since these methods reconstruct the entire graph structure, they often incur a quadratic cost in the number of nodes. While \ours also refines the graph structure by introducing auxiliary connections via VNs, it does not aim to reconstruct the adjacency matrix. Instead, \ours is designed to mitigate the limitations of MPNNs by facilitating dynamic message passing, while ensuring linear cost in the number of nodes.

% alter the adjacency matrix for graph reconstruction
% These methods typically focus on reconstructing the entire graph structure, often incurring a quadratic cost in the number of nodes.

\paragraph{Modified Message Passing}
Several works modify how messages are exchanged between neighboring nodes~\cite{amp}. For example, \cognn~\cite{cognn} allows each node to independently decide whether to aggregate, propagate, both, or neither at each layer, while \nba~\cite{nba} eliminates the self-information from neighbor messages to refine aggregation. \revgnn~\cite{revgnn} proposes a reverse process of message passing, i.e., an inversion of the aggregation operation, to compute more distinguishable representations. However, because these methods operate on the fixed original graph structure, they suffer from inefficient communication between distant nodes, as gathering information from $L$-hop neighbors inevitably requires $L$ layers.

\begin{figure*}[t]
\centering
\includegraphics[width=0.92\textwidth]{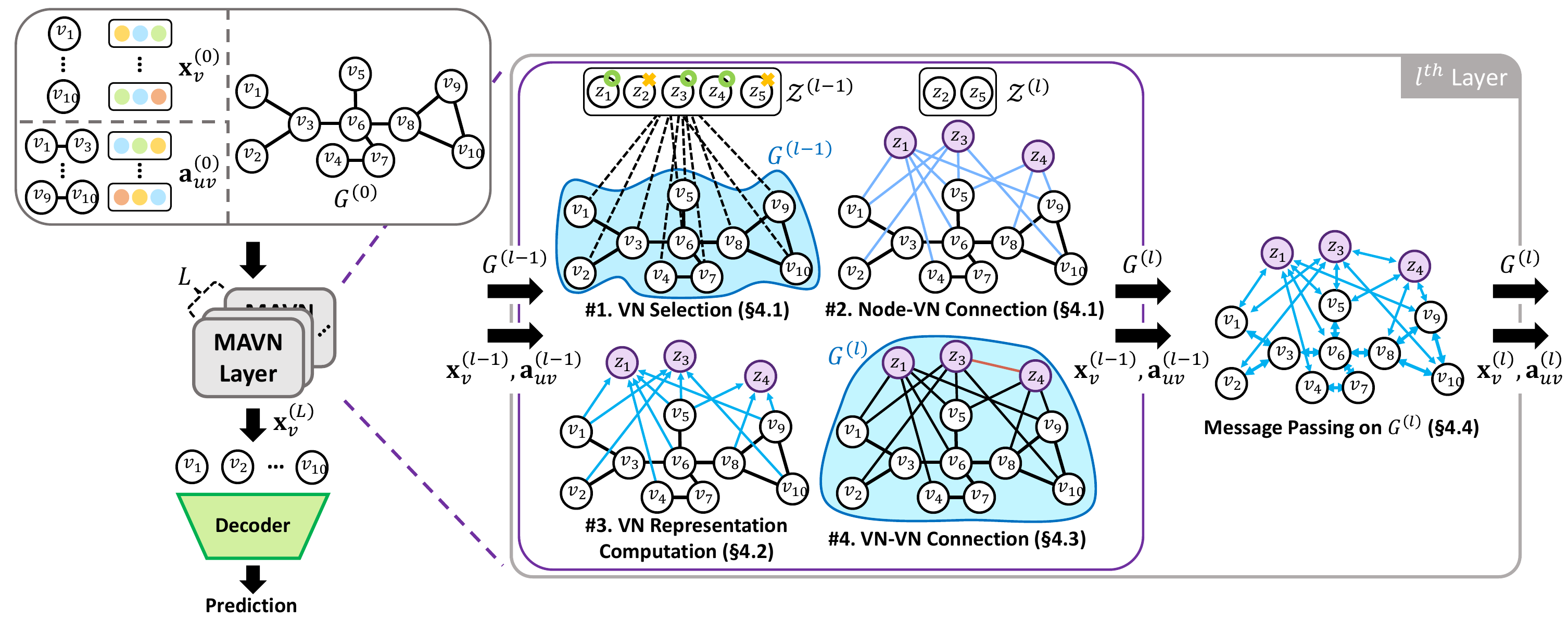}
\caption{Overview of \ours. At the $l$-th layer, \ours selects VNs from $\sZ^{(l-1)}$, adds them to $G^{(l-1)}$ ($\textbf{\#}\mathbf{1}$), and makes connections between the VNs and nodes ($\textbf{\#}\mathbf{2}$). VNs' representations are computed by aggregating their neighboring nodes' representations ($\textbf{\#}\mathbf{3}$). VNs are selectively connected to other VNs ($\textbf{\#}\mathbf{4}$). Once $G^{(l-1)}$ changes to $G^{(l)}$ with added VNs, node-VN connections, and VN-VN connections, an MPNN layer is applied to update representations, which are passed to the subsequent layer.
}
\label{fig:ours}
\end{figure*}

\section{Preliminaries}
Given an undirected graph $G=(\sV,\sE)$, where $\sV$ is a set of nodes and $\sE\subseteq\sV\times\sV$ is a set of edges, each node $v\in\sV$ and edge $(u, v)\in\sE$ has feature vectors $\vv\in\RR^{d'}$ and $\ve_{uv}\in\RR^{d''}$, respectively, where $d'$ and $d''$ denote the dimensions of the node and edge feature vectors, respectively. Let $G^{(l)}=(\sV^{(l)}, \sE^{(l)})$ denote the graph at the $l$-th layer, where $l=1,\cdots,L$, and $L$ is the number of layers. $G^{(0)}=(\sV^{(0)},\sE^{(0)})$ denotes the input graph.

In MPNNs, node representations are updated by aggregating the messages of neighboring nodes, while edge representations are updated using incident nodes. Let $d_l$ denote the dimension at the $l$-th layer. Given $G=(\sV, \sE)$, the $l$-th layer of an MPNN computes $\vx_v^{(l)}\in\RR^{d_l}$ of a node $v\in\sV$ and $\va_{uv}^{(l)}\in\RR^{d_l}$ of an edge $(u,v)\in\sE$ as:
\begin{gather}
    \vx_{v}^{(l)}=\texttt{UPD}_{\texttt{NODE}}^{(l)}\big( \vx_{v}^{(l-1)}, \texttt{AGGR}^{(l)}\big(\{(\vx_{u}^{(l-1)}, \va_{uv}^{(l)}) | u\in\sN_v \}\big)\big),\\
    \va_{uv}^{(l)}=\texttt{UPD}_{\texttt{EDGE}}^{(l)}\big( \va_{uv}^{(l-1)}, \vx_{u}^{(l-1)}, \vx_{v}^{(l-1)} \big),
\end{gather}
where $\vx_v^{(0)}, \va_{uv}^{(0)}\in\RR^{d_0}$ are computed using their feature vectors $\vv$ and $\ve_{uv}$, respectively, and $d_0$ is the dimension. $\texttt{AGGR}^{(l)}$ is an aggregation function at the $l$-th layer and $\texttt{UPD}_{\texttt{NODE}}^{(l)}$ and $\texttt{UPD}_{\texttt{EDGE}}^{(l)}$ are update functions for nodes and edges at the $l$-th layer, respectively. \gcn~\cite{gcn}, \gat~\cite{gat}, and \sage~\cite{gsage} do not use edge representations. \gine~\cite{gine} utilizes the edge representations without updating them, and \gatedgcn~\cite{gatedgcn} uses and updates them.

\section{\ours: Adaptive Virtual Nodes for Dynamic Message Passing}
\label{sec:ours}

Figure~\ref{fig:ours} provides an overview of \ours, which iteratively expands $G^{(l-1)}$ to $G^{(l)}$ for $l=1,\cdots,L$ by introducing VNs and adding edges between VNs and nodes (\S~\ref{subsec:w2vn}), computing VNs' representations using their neighboring nodes' representations (\S~\ref{subsec:vnrep}), and adding connections between VNs (\S~\ref{subsec:vnvn}). Once $G^{(l-1)}$ changes to $G^{(l)}$, an MPNN layer is applied to update representations, which are fed into the subsequent layer (\S~\ref{subsec:dmp}). \ours learns to determine how many and which VNs should be introduced at each layer and which node-VN and VN-VN connections should be formed.

\subsection{Layer-wise Selection of Virtual Nodes and Forming Node-VN Connections}
\label{subsec:w2vn}

Given {$G^{(0)}\!=\!(\sV^{(0)},\sE^{(0)})$}, \ours begins with a pool of {$M$} candidate VNs, denoted by {$\sZ^{(0)}$}, where the hyperparameter {$M$} indicates the budget for the total number of VNs. At the {$l$}-th layer, where {$l=1,\cdots,L$}, \ours selects VNs from {$\sZ^{(l-1)}$} based on the graph-level preference scores for the candidate VNs, which indicates how much the current graph {$G^{(l-1)}$} prefers the corresponding VN. The selected VNs are removed from {$\sZ^{(l-1)}$}, and the remaining VNs comprise {$\sZ^{(l)}$}. In this process, \ours can select a varying number of VNs at each layer, including none. Once selected, the VNs are connected to a nonempty subset of nodes by computing each connection's score based on a dual-perspective scoring, which considers both the node-level and VN-level preference scores. Each node possibly connects to different numbers of VNs, including none. All the graph-level, node-level, and VN-level preference scores are derived from relevance scores between nodes and VNs, as described below.

\paragraph{Relevance Score Computation} The relevance score {$s_{vz}^{(l)}$} between a node {$v\in\sV^{(l-1)}$} and a candidate VN {$z\in\sZ^{(l-1)}$} is computed using the scaled dot-product as follows:
\begin{equation}
\label{eq:rel_score}
s_{vz}^{(l)} = \text{MLP}^{(l)}(\vx_{v}^{(l-1)}) \cdot \vk_z^{(l)} / \sqrt{d_\text{dot}},
\end{equation}
where $\text{MLP}^{(l)}:\RR^{d_{l-1}}\!\to\!\RR^{d_\text{dot}}$ denotes a multilayer perceptron, {$d_\text{dot}$} denotes the dimension of the dot product, and {{$\vk_z^{(l)}\!\in\!\RR^{d_\text{dot}}$}} is a learnable key vector for a VN {$z$} at the $l$-th layer. The relevance score between a node and a VN reflects their pairwise affinity, independent of the other nodes and candidate VNs. Instead of directly utilizing this score for selection processes, we introduce a score adjustment mechanism that calibrates the relevance scores based on their relative importance across all candidates, deriving scores conditioned on the set of candidate connections.

\paragraph{VN Selection} The connection score {$\bar{s}_{vz}^{(l)}$} of a node-VN pair {$(v,z)$} is computed by adjusting the relevance score {${s}_{vz}^{(l)}$} with its relative significance among all node-VN pairs:
\begin{equation}
    \bar{s}_{vz}^{(l)} = s_{vz}^{(l)} + \alpha\texttt{logsoftmax}(\sS^{(l)})[s_{vz}^{(l)}],
\end{equation}
where hyperparameter $\alpha\geq0$ controls the balance between the relevance score and its relative significance, $\sS^{(l)}\!=\!\{ s_{v'z'}^{(l)}|v'\in\sV^{(l-1)},$ $z'\in\sZ^{(l-1)}\}$ is the set of relevance scores of all node-VN pairs, and $\texttt{logsoftmax}(\sS)[s]$ represents the log value of $s\in\sS$ after applying a softmax operation over $\sS$. Since the output of $\texttt{logsoftmax}$ is non-positive, this term acts purely as a penalty. Consequently, the adjusted score balances the absolute pairwise affinity with the relative significance, effectively suppressing connections that are insignificant within the candidate set. This adjustment mechanism is applied throughout subsequent selection stages using candidate sets specific to each selection process, thereby effectively conditioning scores on the candidates and pruning out unnecessary connections while establishing significant ones.

The graph-level preference score for a VN $z$, denoted by $s_{z}^{(l)}$, is computed by aggregating the connection scores of all node-VN pairs involving $z$, with the $\texttt{logmeanexp}$ function:
\begin{equation}
    s_{z}^{(l)}=\texttt{logmeanexp}(\{\bar{s}_{vz}^{(l)}|v\in\sV^{(l-1)}\}),
\end{equation}
where $\texttt{max}(\sS)-\texttt{log}(|\sS|)\!\leq\!\texttt{logmeanexp}(\sS)\!\!=\!\text{log}(\sum_{s\in\sS}\text{exp}(s)/|\sS|)$ $\leq\texttt{max}(\sS)$ is a smooth approximation of the maximum that enables gradients to flow through all nodes. Using the graph-level preference scores for VNs, \ours selects {$z$} such that $\texttt{sigmoid}(s_{z}^{(l)}) \geq 0.5$. Let $\shV^{(l)}$ denote the set of VNs selected at the {$l$}-th layer. Notably, for each VN $z'\in\shV^{(l)}$, there exists at least one node $v'$ such that $\bar{s}_{v'z'}^{(l)}\geq0$, a property derived from the upper bound of $\texttt{logmeanexp}$.

\paragraph{Connecting Nodes and VNs} When $\shV^{(l)}\neq\emptyset$, \ours decides whether an edge between a node $v$ and a VN $z$, i.e., $(v,z)\in\sV^{(l-1)}\times\shV^{(l)}$, should be formed or not, based on a dual-perspective scoring that scores $(v,z)$ from both $z$'s and $v$'s point of view using two preference scores: a VN-level preference score $\tilde{s}_{vz|z}^{(l)}$, computed by adjusting the relevance score based on how significant $v$ is for $z$ compared to all nodes in $\sV^{(l-1)}$, and a node-level preference score $\tilde{s}_{vz|v}^{(l)}$, computed by adjusting the relevance score based on how significant $z$ is for $v$ compared to all VNs in $\shV^{(l)}$. These preference scores are combined to yield the final edge score $\tilde{s}_{vz}^{(l)}$:
\begin{equation}
    \tilde{s}_{vz}^{(l)}=\beta^{(l)}\cdot\tilde{s}_{vz|z}^{(l)}+(1-\beta^{(l)})\cdot \tilde{s}_{vz|v}^{(l)},
\end{equation}
where $\tilde{s}_{vz|z}^{(l)}\!=\!s_{vz}^{(l)}\!+\!\alpha\texttt{logsoftmax}(\sS^{(l)}_{:z})[s_{vz}^{(l)}]$ and $\tilde{s}_{vz|v}^{(l)}\!=\!s_{vz}^{(l)}\!+\!\alpha\texttt{log}$-$\texttt{softmax}(\sS^{(l)}_{v:})[s_{vz}^{(l)}]$, $\sS^{(l)}_{:z}\!\!=\!\{ s_{v'z}^{(l)}|v'\in\sV^{(l-1)}\}$ is the set of relevance scores for all candidate edges between {$z$} and the nodes in $\sV^{(l-1)}$, $\sS^{(l)}_{v:}\!=\!\{ s_{vz'}^{(l)} | z'\!\!\in\!\shV^{(l)}\}$ is the set of relevance scores for all candidate edges between {$v$} and all VNs in $\shV^{(l)}$, and $\beta^{(l)}\in[0,1]$ is a learnable scalar that controls the balance between the VN-level and node-level scores. An edge between $v$ and $z$ is formed if $\texttt{sigmoid}(\tilde{s}_{vz}^{(l)})\!\geq\!0.5$, and its representation $\va^{(l-1)}_{vz}$ is set to be a learnable vector $\vha^{(l)}_\text{N-VN}\!\in\!\mathbb{R}^{d_{l-1}}$ shared across all edges in $\shE^{(l)}_\text{N-VN}$, the set of added node-VN edges at the {$l$}-th layer. Note that both $\texttt{logsoftmax}(\sS^{(l)}_{:z})[s_{vz}^{(l)}]$ and $\texttt{logsoftmax}($ $\sS^{(l)}_{v:})[s_{vz}^{(l)}]$ are greater than or equal to $\texttt{logsoftmax}(\sS^{(l)})[s_{vz}^{(l)}]$, since $\sS^{(l)}_{:z}$ and $\sS^{(l)}_{v:}$ are subsets of $\sS^{(l)}$. This relationship ensures that $\tilde{s}_{vz}^{(l)}\geq\bar{s}_{vz}^{(l)}$; thus, every selected VN $z'\in\shV^{(l)}$ is guaranteed to be connected to at least one node, as there exists a node $v'$ satisfying $\bar{s}_{v'z'}^{(l)}\geq0$. 

\subsection{Computing Virtual Nodes' Representations via Aggregating Nodes' Representations}
\label{subsec:vnrep}
A representation $\vx_{z}^{(l-1)}\in\RR^{d_{l-1}}$ of a VN $z\in\shV^{(l)}$ is computed by a weighted combination of its learnable seed representation $\vq_z^{(l)}\in\RR^{d_{l-1}}$ and an aggregated representation of its connected nodes with a learnable gating vector $\boldsymbol{\gamma}^{(l)}\in[0,1]^{d_{l-1}}$:
\begin{equation}
        \vx_{z}^{(l-1)}=\boldsymbol{\gamma}^{(l)}\odot\vq_z^{(l)}
    +(\mathbf{1}_{d_{l-1}}-\boldsymbol{\gamma}^{(l)})\odot\left(\sum\nolimits_{v \in \sN_{z}^{(l)}}\frac{p_{vz}^{(l)}\vx_{v}^{(l-1)}}{c_z^{(l)}}\right),\\
\end{equation}
where $\mathbf{1}_{d_{l-1}}\in\mathbb{R}^{d_{l-1}}$ is a $d_{l-1}$-dimension vector of ones, {$\odot$} denotes the Hadamard product, $\sN_{z}^{(l)}\!=\!\{v|(v,z)\!\in\!\shE^{(l)}_\text{N-VN}\}$ is the set of nodes connected to $z$, $p_{vz}^{(l)}=\texttt{sigmoid}(\tilde{s}_{vz}^{(l)})$, and $c_z^{(l)}$ determines the aggregation type: $c_z^{(l)}=\sum_{v \in \sN_{z}^{(l)}}{ p_{vz}^{(l)}}$ indicates a weighted mean, whereas $c_z^{(l)}=1$ indicates a weighted sum. This aggregation is performed directly without applying additional linear projections, ensuring that the derived VN representations align with the feature space of the node representations.

\subsection{Selective Connections Between VNs}
\label{subsec:vnvn} Edges between VNs in $\shV^{(l)}$ are formed using a procedure similar to that used for node-VN connections. A relevance score between two VNs {$z$} and {$u$} is computed by $s_{zu}^{(l)}=\text{MLP}^{(l)}\big(\vx_{z}^{(l-1)}\big)\cdot\text{MLP}^{(l)}\big(\vx_{u}^{(l-1)}\big)/\sqrt{d_\text{dot}}$, where $\text{MLP}^{(l)}$ is the one in Eq.~\ref{eq:rel_score}. Then, we consider both VNs' viewpoints, $\bar{s}_{zu|z}^{(l)}\!=\!s_{zu}^{(l)}\!+\!\alpha\texttt{logsoftmax}(\{s_{zu'}^{(l)}|u'\!\in\!$ $\shV^{(l)}\backslash\{z\}\})[s_{zu}^{(l)}]$ and $\bar{s}_{zu|u}^{(l)}\!=\!s_{zu}^{(l)}+\alpha\texttt{logsoftmax}(\{ s_{z'u}^{(l)}|z'\!\in\!\shV^{(l)}\backslash\{u\}$ $\})[s_{zu}^{(l)}]$, to obtain the final edge score $\bar{s}_{zu}^{(l)}\!=\!\frac{1}{2}{(\bar{s}_{zu|z}^{(l)}\!+\!\bar{s}_{zu|u}^{(l)})}$. Since both {$z$} and {$u$} are VNs, we equally weigh both perspectives in computing $\bar{s}_{zu}^{(l)}$. This scoring formulation jointly evaluates the affinity between VNs via $s_{zu}^{(l)}$ and their relative significance compared to other candidates via the \texttt{logsoftmax} term. An edge between {$z$} and {$u$} is formed if $\texttt{sigmoid}(\bar{s}_{zu}^{(l)})\!\geq\!0.5$, and its representation is set to a learnable vector $\vha^{(l)}_\text{VN-VN}\in \mathbb{R}^{d_{l-1}}$ shared across all VN-VN edges in $\shE_\text{VN-VN}^{(l)}$, the set of added VN-VN edges at the {$l$}-th layer.

\subsection{Dynamic Message Passing with Adaptively Introduced Virtual Nodes}\label{subsec:dmp}
At the {$l$}-th layer, $G^{(l-1)}$ changes to $G^{(l)}=(\sV^{(l)}, \sE^{(l)})=(\sV^{(l-1)}\cup\shV^{(l)}, \sE^{(l-1)}\cup\shE^{(l)})$, where $\shV^{(l)}$ is the set of VNs selected at the {$l$}-th layer (\S~\ref{subsec:w2vn}) and $\shE^{(l)}$ is the set of edges added at the {$l$}-th layer. Note that $\shE^{(l)}$ consists of $\shE^{(l)}_\text{N-VN}$ (\S~\ref{subsec:w2vn}) and $\shE^{(l)}_\text{VN-VN}$ (\S~\ref{subsec:vnvn}). Given $G^{(l)}$, an MPNN layer computes representations at the {$l$}-th layer, which are passed to the subsequent layer along with $G^{(l)}$. In this way, $G^{(0)}$ dynamically and iteratively changes to $G^{(l)}$ for {$l=1,\cdots,L$}, where VNs are adaptively introduced when needed and connected to nodes and other VNs, as many as needed at each layer. \ours only utilizes node representations and does not assume any specific MPNN architecture, so it is architecture-agnostic and can be seamlessly applied to any MPNNs.

\begin{figure*}[t]
\begin{minipage}[t]{0.64\textwidth}\vspace{0pt}
\begin{subfigure}{0.48\textwidth}
    \centering
    \includegraphics[width=\linewidth]{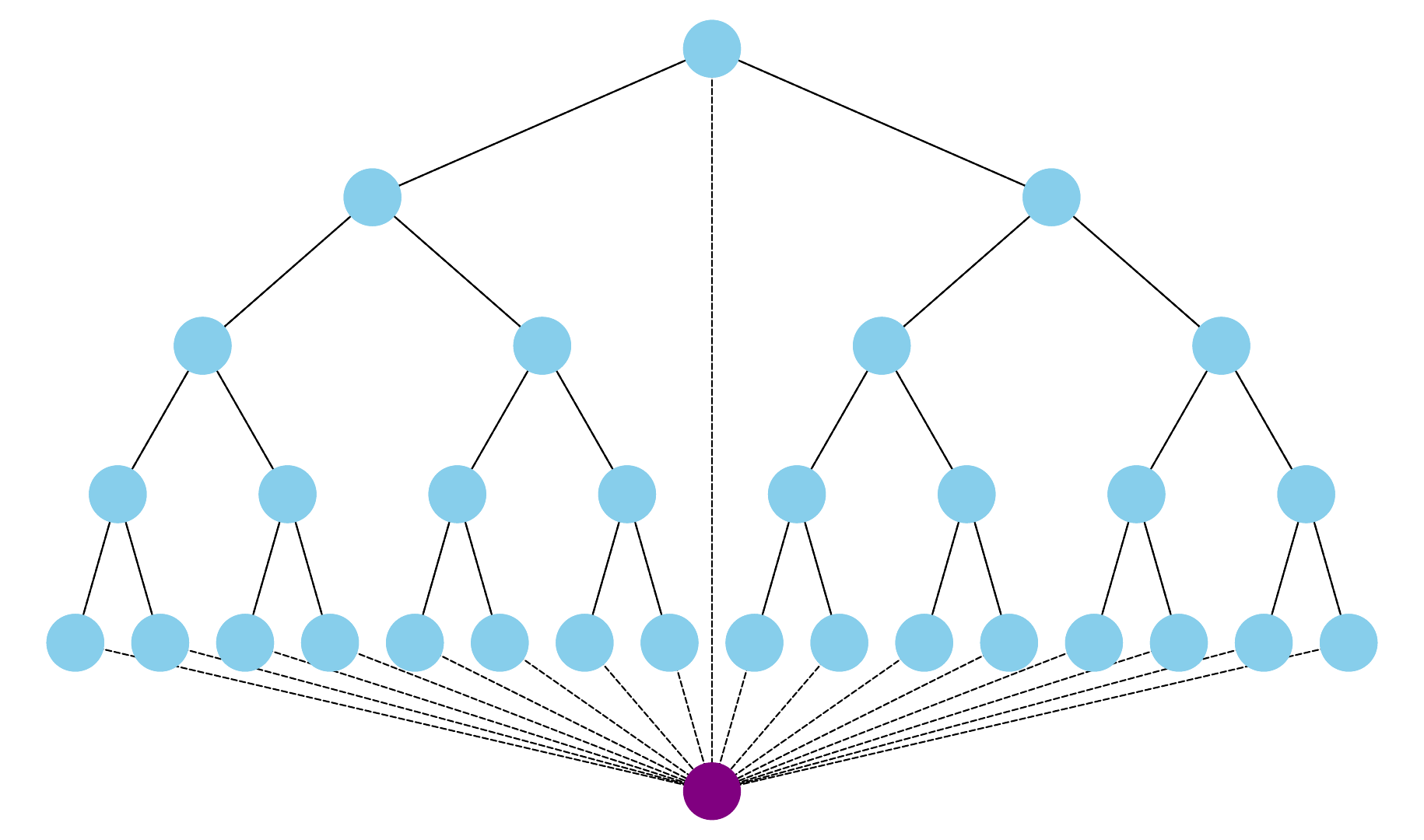}
    \caption{Under-reaching}
    \label{fig:undr}
\end{subfigure}
\begin{subfigure}{0.48\textwidth}
\centering
\includegraphics[width=\linewidth]{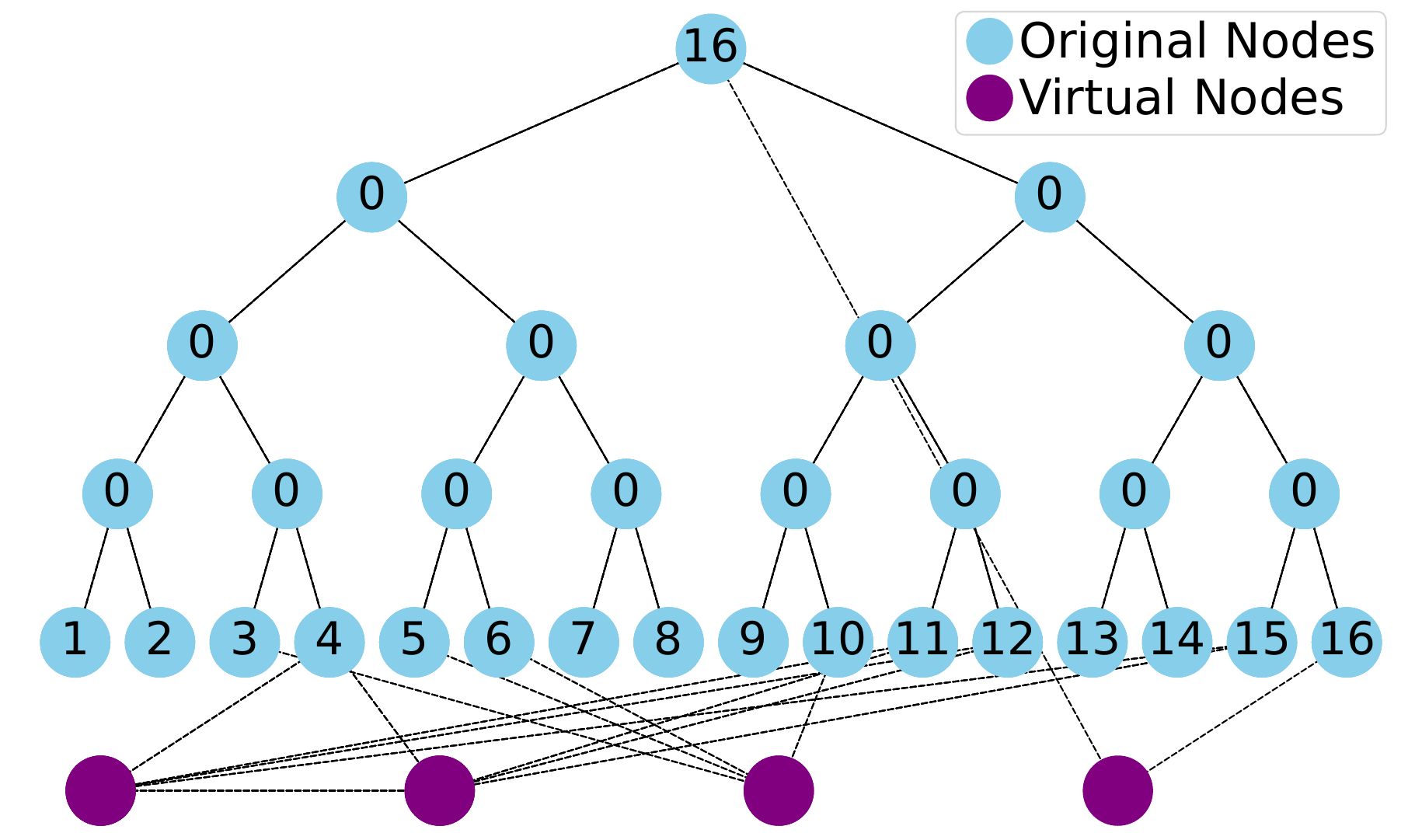}
\caption{Over-squashing}
\label{fig:ovsq}
\end{subfigure}
\caption{Examples of how \ours mitigates under-reaching and over-squashing. Purple nodes indicate VNs introduced by \ours.}
\end{minipage}\hfill
\begin{minipage}[t]{0.34\textwidth}\vspace{0pt}
\captionof{table}{Classification accuracy on two synthetic tree datasets with depths 4 to 6.}
\centering
\begin{tabular}{cccc}
\toprule
& depth & GCN & \ours\\
\midrule
 & 4 & 0.14 & 1.00 \\
\treelc & 5 & 0.11 & 1.00 \\
 & 6 & 0.08 & 1.00 \\
 \midrule
 & 4 & 0.54 & 1.00 \\
\treenb & 5 & 0.22 & 1.00 \\
 & 6 & 0.10 & 1.00 \\
\bottomrule
\end{tabular}
\label{tb:syn}
\end{minipage}
\end{figure*}

\paragraph{Training Details} Since the selection of VNs and edges involves discrete thresholding operations, we use the binary Gumbel Softmax~\cite{gsoft} with a straight-through estimator~\cite{st} to enable gradient propagation. \ours is trained using only a task-specific loss, without any additional losses for controlling the behavior of VNs. More implementation details are provided in Appendix~\ref{app:impl}.

\paragraph{Multi-head Strategy} We adopt a multi-head strategy~\citep{trf} in \S~\ref{subsec:w2vn} and~\ref{subsec:vnvn}. Specifically, each head computes relevance scores independently (\S~\ref{subsec:w2vn}), with dot-product dimension $d_\text{dot}/h$, where $h$ denotes the number of heads. The $\texttt{logsoftmax}$ operation is applied per head to yield the per-head connection score (\S~\ref{subsec:w2vn},~\ref{subsec:vnvn}), the per-head node-VN edge score (\S~\ref{subsec:w2vn}), and the per-head VN-VN edge score (\S~\ref{subsec:vnvn}). These per-head scores are averaged across all heads to produce the final connection score, node-VN edge score, and VN-VN edge score, respectively. The final connection scores are then aggregated over nodes to compute a graph-level preference score for each VN. The graph-level preference scores of VNs, node-VN edge scores, and VN-VN edge scores are each passed through a \texttt{sigmoid} function for selection.

\paragraph{Complexity Analysis} In \ours, \S~\ref{subsec:w2vn} and \S~\ref{subsec:vnrep} require $\sO(M|\sV|)$. \S~\ref{subsec:vnvn} introduces an additional complexity of $\sO(M^2)$. With at most $M|\sV|\!+\!M^2$ new edges, \S~\ref{subsec:dmp} takes $\sO(|\sE|\!+\!M|\sV|\!+\!M^2)$. Thus, the total complexity of \ours is $\sO(L(|\sE|\!\!+\!M|\sV|\!+\!M^2))$, which is asymptotically equivalent to \iprmpnn and \nsquare. In practice, $M$ is typically set to be much smaller than $|\sV|$, which suppresses the quadratic term $M^2$ and yields an overall complexity of $\sO(L(|\sE|\!+\!M|\sV|))$.

\section{Properties of \ours}
\label{sec:propours}
We prove that for any set of message passing paths involving VNs, there exists a parameter configuration of \ours that constructs those paths. Additionally, experiments on synthetic datasets validate that \ours mitigates under-reaching and over-squashing by appropriately introducing VNs.

\subsection{Flexibility of Message Passing Path Generation in \ours}
\label{subsec:flex}
Theorem~\ref{thm:flexours} shows that a single-layer \ours is capable of constructing any message passing paths that involve VNs.
\begin{theorem}
\label{thm:flexours}
Given a graph {$G=(\sV,\sE)$} with {$K$} sets of nodes {$\sV_1, \sV_2, \cdots, \sV_K \subseteq \sV$}, if node representations are uniquely distinguishable (i.e., {$\vx_u=\vx_v\iff u=v, \forall u,v\in \sV$}), there exists a parameter configuration of a single-layer \ours that introduces {$K$} virtual nodes {$z_1, z_2, \cdots, z_K$} into {$G$}, where each {$z_i$} is connected to a node {$v\in\sV$} if and only if {$v\in\sV_i$}.
\end{theorem}

The proof of Theorem~\ref{thm:flexours} is in Appendix~\ref{app:prf}, derived by using the universal approximation theorem for MLPs~\cite{univ}. Note that node representations can be distinguished by incorporating an injective positional encoding~\cite{guniv}. Theorem~\ref{thm:flexours} confirms that \ours can generate all necessary message passing paths simultaneously at each layer. Furthermore, Corollary~\ref{cor:gen} shows that \ours can simulate any message passing path construction strategies that connect VNs and nodes, including those employed by existing VN-based methods such as \iprmpnn~\cite{iprmpnn} and \nsquare~\cite{nsquare}.

\begin{corollary}
\label{cor:gen}
Any graph structure obtained by connecting VNs to nodes in the original graph, such as those produced by \iprmpnn~\cite{iprmpnn} or \nsquare~\cite{nsquare}, can be simulated by a single-layer \ours.
\end{corollary}
This follows directly from Theorem~\ref{thm:flexours} by defining {$\sV_i$} as the set of nodes connected to the {$i$}-th VN. 

\subsection{Mitigating Under-reaching and Over-squashing via \textbf{\ours}}
To show that \ours can alleviate under-reaching and over-reaching, we conduct experiments on synthetic datasets, using GCN~\cite{gcn} as the backbone MPNN. In Appendix~\ref{app:prop}, we show that \ours does not introduce \textbf{over-smoothing}, evidenced by the Dirichlet energy~\cite{ovsmsurv} of node representations computed by \ours at each layer.

\paragraph{Under-reaching} \ours can address under-reaching by introducing VNs that create shortcuts between distant nodes, enabling information flow beyond the {$L$}-hop neighbors with {$L$}-layer MPNN. We empirically validate it using the \synunder~\cite{prmpnn} (\treelc) dataset, where each node has a binary label, and the task is to predict the number of leaf nodes labeled `1' using only the root node's representation in a single-layer setup. Table~\ref{tb:syn} shows that \ours achieves the perfect accuracy across varying tree depths, whereas the backbone MPNN, GCN, performs at the level of random chance. Figure~\ref{fig:undr} illustrates that \ours connects a VN to the root and all the leaf nodes to mitigate under-reaching.

\paragraph{Over-squashing} Over-squashing occurs when information from distant nodes is overly compressed due to a limited number of paths between them~\cite{ovsqc, ovsqmpnn}. \ours can alleviate over-squashing by introducing additional paths between such node pairs using VNs. We empirically validate this claim on the \synsquash~\cite{ovsq} (\treenb) dataset, where one should predict the value label of a leaf node that shares the same key label as the root node, using the representation of the root node. As shown in Table~\ref{tb:syn}, while GCN fails, \ours achieves the perfect accuracy across trees of depth 4 to 6. Figure~\ref{fig:ovsq} shows that \ours learns to introduce multiple paths between a pair of nodes via VNs to effectively mitigate the over-squashing issue. In Appendix~\ref{app:prop}, we also provide the total effective resistance values~\cite{egr} to further confirm that \ours successfully resolves over-squashing on real-world datasets.

\begin{table*}[t]
\centering
\caption{Experimental results on four LRGB datasets. We report AP(\%) for \pepfunc, MAE for \pepstruct, and F1-macro(\%) for \pascal and \coco. The results are from the original papers or~\cite{lrgbexp} (marked by $^*$).}
\setlength{\tabcolsep}{1mm}
\setlength\fboxsep{1pt}
\begin{tabular}{lllllc}
\toprule
     & \pepfunc($\uparrow$) & \pepstruct($\downarrow$) & \pascal($\uparrow$) & \coco($\uparrow$) \\
    \cmidrule(lr){1-5}
    \drew~\cite{drew} & 71.50\scalebox{0.6}{\color{gray} $\pm$0.44} & 0.2536\scalebox{0.6}{\color{gray} $\pm$0.0015} & 33.14\scalebox{0.6}{\color{gray} $\pm$0.24} & - \\
    \laser~\cite{laser} & 64.89\scalebox{0.6}{\color{gray} $\pm$0.74} & 0.2971\scalebox{0.6}{\color{gray} $\pm$0.0037} & - & - \\
    \prmpnn~\cite{prmpnn} & 68.25\scalebox{0.6}{\color{gray} $\pm$0.86} & 0.2477\scalebox{0.6}{\color{gray} $\pm$0.0005} & - & - \\
     \megraph~\cite{megraph} & 69.45\scalebox{0.6}{\color{gray} $\pm$0.77} & 0.2507\scalebox{0.6}{\color{gray} $\pm$0.0009} & - & - \\
     MPNN+VN~\cite{guniv} & 68.23\scalebox{0.6}{\color{gray} $\pm$0.69} & 0.2475\scalebox{0.6}{\color{gray} $\pm$0.0018} & \textit{44.77}\scalebox{0.6}{\color{gray} $\pm$1.37} & 32.44\scalebox{0.6}{\color{gray} $\pm$0.25} \\
     MPNN+VN$_G$~\cite{vntheory}  & 68.22\scalebox{0.6}{\color{gray} $\pm$0.52} & 0.2458\scalebox{0.6}{\color{gray} $\pm$0.0006} & - & - \\
     \iprmpnn~\cite{iprmpnn} & \textit{72.10}\scalebox{0.6}{\color{gray} $\pm$0.39} & \underline{0.2422}\scalebox{0.6}{\color{gray} $\pm$0.0007} & - & - \\
     \geaet~\cite{geaet} & - & 0.2445\scalebox{0.6}{\color{gray} $\pm$0.0013} & \underline{45.85}\scalebox{0.6}{\color{gray} $\pm$0.87} & \textbf{38.95}\scalebox{0.6}{\color{gray} $\pm$0.50} \\
      \stgcn~\cite{stgcn} & \textbf{73.11}\scalebox{0.6}{\color{gray} $\pm$0.66} & 0.2447\scalebox{0.6}{\color{gray} $\pm$0.0032} & - & - \\
    $\text{\gps}^*$~\cite{gps} & 65.34\scalebox{0.6}{\color{gray} $\pm$0.91} & 0.2509\scalebox{0.6}{\color{gray} $\pm$0.0014} & 44.40\scalebox{0.6}{\color{gray} $\pm$0.65} & \underline{38.84}\scalebox{0.6}{\color{gray} $\pm$0.55} \\
    \exphormer~\cite{exphormer} & 65.27\scalebox{0.6}{\color{gray} $\pm$0.43} & 0.2481\scalebox{0.6}{\color{gray} $\pm$0.0007} & 39.75\scalebox{0.6}{\color{gray} $\pm$0.37} & 34.55\scalebox{0.6}{\color{gray} $\pm$0.09} \\
    \grit~\cite{grit} & 69.88\scalebox{0.6}{\color{gray} $\pm$0.82} & 0.2460\scalebox{0.6}{\color{gray} $\pm$0.0012} & - & - \\
    \gvit~\cite{gvitmixer} & 69.42\scalebox{0.6}{\color{gray} $\pm$0.75} & 0.2449\scalebox{0.6}{\color{gray} $\pm$0.0016} & - & -\\
     \amp~\cite{amp} & 71.61\scalebox{0.6}{\color{gray} $\pm$0.47} & 0.2446\scalebox{0.6}{\color{gray} $\pm$0.0026} & - & - \\
     \cognn~\cite{cognn} & 69.90\scalebox{0.6}{\color{gray} $\pm$0.93} & - & - & - \\
     \nba~\cite{nba} & 72.07\scalebox{0.6}{\color{gray} $\pm$0.28} & \textit{0.2424}\scalebox{0.6}{\color{gray} $\pm$0.0010} & 39.69\scalebox{0.6}{\color{gray} $\pm$0.27} & - \\
     \uniconv~\cite{uniconv} & 71.73\scalebox{0.6}{\color{gray} $\pm$0.61} & 0.2425\scalebox{0.6}{\color{gray} $\pm$0.0009} & 40.05\scalebox{0.6}{\color{gray} $\pm$0.67} & 31.53\scalebox{0.6}{\color{gray} $\pm$0.35} \\
    \cmidrule(lr){1-5}
     $\text{\gcn}^*$~\cite{gcn} & 68.60\scalebox{0.6}{\color{gray} $\pm$0.50} & 0.2460\scalebox{0.6}{\color{gray} $\pm$0.0007} & 20.78\scalebox{0.6}{\color{gray} $\pm$0.31} & 13.38\scalebox{0.6}{\color{gray} $\pm$0.07} \\
     $\text{\gine}^*$~\cite{gine} & 66.21\scalebox{0.6}{\color{gray} $\pm$0.67} & 0.2473\scalebox{0.6}{\color{gray} $\pm$0.0017} & 27.18\scalebox{0.6}{\color{gray} $\pm$0.54} & 21.25\scalebox{0.6}{\color{gray} $\pm$0.09} \\
    $\text{\gatedgcn}^*$~\cite{gatedgcn} & 67.65\scalebox{0.6}{\color{gray} $\pm$0.47} & 0.2477\scalebox{0.6}{\color{gray} $\pm$0.0009} & 38.80\scalebox{0.6}{\color{gray} $\pm$0.40} & 29.22\scalebox{0.6}{\color{gray} $\pm$0.18} & avg. impr.\\
    \cmidrule(lr){1-5} \cmidrule(lr){6-6}
    \ours-\gcn & \underline{72.38}\scalebox{0.6}{\color{gray} $\pm$0.52} \colorbox{lightgray}{(+5.5\%)} & 0.2432\scalebox{0.6}{\color{gray} $\pm$0.0014} \colorbox{lightgray}{(+1.2\%)} & 30.44\scalebox{0.6}{\color{gray} $\pm$0.35} \colorbox{lightgray}{(+46.5\%)} & 18.72\scalebox{0.6}{\color{gray} $\pm$1.03} \colorbox{lightgray}{(+39.9\%)} & +23.3\%\\
     \ours-\gine & 68.97\scalebox{0.6}{\color{gray} $\pm$0.50} \colorbox{lightgray}{(+4.1\%)} & 0.2453\scalebox{0.6}{\color{gray} $\pm$0.0016} \colorbox{lightgray}{(+0.8\%)} & 35.94\scalebox{0.6}{\color{gray} $\pm$0.40} \colorbox{lightgray}{(+32.2\%)} & 24.43\scalebox{0.6}{\color{gray} $\pm$0.34} \colorbox{lightgray}{(+15.0\%)} & +13.0\% \\
     \ours-\gatedgcn & 70.70\scalebox{0.6}{\color{gray} $\pm$0.28} \colorbox{lightgray}{(+4.5\%)} & \textbf{0.2410}\scalebox{0.6}{\color{gray} $\pm$0.0007} \colorbox{lightgray}{(+2.7\%)} & \textbf{48.28}\scalebox{0.6}{\color{gray} $\pm$0.54} \colorbox{lightgray}{(+24.4\%)} & \textit{37.41}\scalebox{0.6}{\color{gray} $\pm$0.49} \colorbox{lightgray}{(+28.0\%)} & +14.9\%\\
    \bottomrule
\end{tabular}
\label{tb:mainexp_lrgb}
\end{table*}

\section{Experiments}
\label{sec:exp}
On nine real-world datasets, we compare \ours with state-of-the-art methods, including graph rewiring methods, graph transformers, and VN-based methods. The best performance is \textbf{boldfaced}, the second-best is \underline{underlined}, and the third-best is \textit{italicized}. ``-'' denotes that the results are unavailable from the baselines' original papers. We report the mean and the standard deviation of the performance. The relative improvement of \ours over each backbone MPNN is reported in {\setlength\fboxsep{1pt}\colorbox{lightgray}{highlighted}} parentheses. The “avg. impr.” is the average improvement by \ours for each backbone across all datasets. Additional experimental and dataset details are in Appendix~\ref{app:exp} and Appendix~\ref{app:data}, respectively.

\subsection{Performance on Benchmark Graph Datasets}

\paragraph{Performance on LRGB} We use four multi-graph datasets from the Long Range Graph Benchmark (LRGB)~\cite{lrgb}: \pepfunc for graph classification, \pepstruct for graph regression, and \coco and \pascal for node classification. All evaluations are conducted under an inductive setting, where the training and test graphs are disjoint. We follow the standard evaluation setting~\cite{lrgb}. We adopt \gcn~\cite{gcn}, \gine~\cite{gine}, and \gatedgcn~\cite{gatedgcn} as \ours's backbone MPNNs, as they are commonly used in this benchmark. As shown in Table~\ref{tb:mainexp_lrgb}, \ours outperforms all baselines on \pepstruct and \pascal, while achieving second-best on \pepfunc and third-best on \coco. When scoping down to MPNN-based methods, \ours performs best on all datasets except \pepfunc. Across all datasets, \ours consistently improves the corresponding backbone's performance: \gcn by 23.3\%, \gine by 13.0\%, and \gatedgcn by 14.9\% on average, resulting in an overall average improvement of 17.1\%. These results suggest that \ours can generate message passing paths that facilitate long-range interactions.

\begin{table*}[t]
\centering
\caption{Experimental results on heterophilic graphs. We report accuracy({$\%$}) for \empire and \amazon and AUC-ROC({$\%$}) for \minesweeper, \tolokers, and \questions. The backbones' results are reproduced; the rest are from the original papers.}
\setlength{\tabcolsep}{1mm}
\setlength\fboxsep{1pt}
\begin{tabular}{llllllc}
    \toprule
     &\empire($\uparrow$) & \amazon($\uparrow$) & \minesweeper($\uparrow$) & \tolokers($\uparrow$) & \questions($\uparrow$) \\
    \cmidrule(lr){1-6}
    \comfy~\cite{comfy} & 79.53\scalebox{0.6}{\color{gray} $\pm$0.70} & 49.45\scalebox{0.6}{\color{gray} $\pm$0.70} & 89.76\scalebox{0.6}{\color{gray} $\pm$0.50} & - & - \\
    \jdr~\cite{jdr} & 78.86\scalebox{0.6}{\color{gray} $\pm$0.48} & 46.47\scalebox{0.6}{\color{gray} $\pm$0.67} & 90.01\scalebox{0.6}{\color{gray} $\pm$0.32} & 84.73\scalebox{0.6}{\color{gray} $\pm$0.45} & 77.52\scalebox{0.6}{\color{gray} $\pm$0.63} \\
     \iprmpnn~\cite{iprmpnn} & 83.90\scalebox{0.6}{\color{gray} $\pm$0.60} & 48.00\scalebox{0.6}{\color{gray} $\pm$0.70} & 88.70\scalebox{0.6}{\color{gray} $\pm$0.60} & 82.00\scalebox{0.6}{\color{gray} $\pm$0.80} & - \\
     \nsquare~\cite{nsquare} & - & 50.25\scalebox{0.6}{\color{gray} $\pm$0.53} & 93.97\scalebox{0.6}{\color{gray} $\pm$0.27} & \textit{86.25}\scalebox{0.6}{\color{gray} $\pm$0.41} & 78.07\scalebox{0.6}{\color{gray} $\pm$0.63} \\
    \polyformer~\cite{polyformer} & 80.27\scalebox{0.6}{\color{gray} $\pm$0.39} & - & 92.02\scalebox{0.6}{\color{gray} $\pm$0.32} & 84.32\scalebox{0.6}{\color{gray} $\pm$0.59} & 78.32\scalebox{0.6}{\color{gray} $\pm$0.67} \\
    \polynormer~\cite{polynormer} & \textbf{92.55}\scalebox{0.6}{\color{gray} $\pm$0.37} & 54.81\scalebox{0.6}{\color{gray} $\pm$0.49} & 97.46\scalebox{0.6}{\color{gray} $\pm$0.36} & 85.91\scalebox{0.6}{\color{gray} $\pm$0.74} & 78.92\scalebox{0.6}{\color{gray} $\pm$0.89} \\
    \cognn~\cite{cognn} & \textit{91.57}\scalebox{0.6}{\color{gray} $\pm$0.32} & 54.17\scalebox{0.6}{\color{gray} $\pm$0.37} & 97.31\scalebox{0.6}{\color{gray} $\pm$0.41} & 84.45\scalebox{0.6}{\color{gray} $\pm$1.17} & \textbf{80.02}\scalebox{0.6}{\color{gray} $\pm$0.86} \\
    \revgnn~\cite{revgnn} & 86.43\scalebox{0.6}{\color{gray} $\pm$0.74} & 52.75\scalebox{0.6}{\color{gray} $\pm$0.62} & 96.05\scalebox{0.6}{\color{gray} $\pm$0.19} & 86.08\scalebox{0.6}{\color{gray} $\pm$0.84} & 77.96\scalebox{0.6}{\color{gray} $\pm$0.96} \\
    \uniconv~\cite{uniconv} & 87.21\scalebox{0.6}{\color{gray} $\pm$0.76} & \underline{55.34}\scalebox{0.6}{\color{gray} $\pm$0.74} & 96.11\scalebox{0.6}{\color{gray} $\pm$0.10} & 85.18\scalebox{0.6}{\color{gray} $\pm$0.43} & \underline{80.01}\scalebox{0.6}{\color{gray} $\pm$0.43} \\
    \cmidrule(lr){1-6}
    \gcn~\cite{gcn} & 90.81\scalebox{0.6}{\color{gray}$\pm$0.37} & 53.38\scalebox{0.6}{\color{gray}$\pm$0.59} & 97.10\scalebox{0.6}{\color{gray}$\pm$0.38} & \underline{86.30}\scalebox{0.6}{\color{gray}$\pm$0.56}& 78.49\scalebox{0.6}{\color{gray}$\pm$1.19} \\
    \gat~\cite{gat} & 90.39\scalebox{0.6}{\color{gray}$\pm$0.36} & \textit{55.12}\scalebox{0.6}{\color{gray}$\pm$0.37} & \underline{98.26}\scalebox{0.6}{\color{gray}$\pm$0.22} & 85.66\scalebox{0.6}{\color{gray}$\pm$0.52} & 77.91\scalebox{0.6}{\color{gray}$\pm$1.07}  \\
     \sage~\cite{gsage} & 89.81\scalebox{0.6}{\color{gray}$\pm$0.35} & 55.04\scalebox{0.6}{\color{gray}$\pm$0.43} & \textit{98.12}\scalebox{0.6}{\color{gray}$\pm$0.31} & 84.82\scalebox{0.6}{\color{gray}$\pm$0.66} & 77.42\scalebox{0.6}{\color{gray}$\pm$0.88} & avg. impr. \\
    \cmidrule(lr){1-6} \cmidrule(lr){7-7}
     \ours-\gcn & 91.64\scalebox{0.6}{\color{gray}$\pm$0.38} \colorbox{lightgray}{(+0.9\%)} & 53.94\scalebox{0.6}{\color{gray}$\pm$0.36} \colorbox{lightgray}{(+1.0\%)} & 97.90\scalebox{0.6}{\color{gray}$\pm$0.34} \colorbox{lightgray}{(+0.8\%)} & 86.64\scalebox{0.6}{\color{gray}$\pm$0.51} \colorbox{lightgray}{(+0.4\%)} & \textit{79.12}\scalebox{0.6}{\color{gray}$\pm$1.08} \colorbox{lightgray}{(+0.7\%)} & +0.8\%\\
     \ours-\gat & \underline{92.30}\scalebox{0.6}{\color{gray}$\pm$0.32} \colorbox{lightgray}{(+2.1\%)} & 55.31\scalebox{0.6}{\color{gray}$\pm$0.37} \colorbox{lightgray}{(+0.3\%)} & 98.70\scalebox{0.6}{\color{gray}$\pm$0.21} \colorbox{lightgray}{(+0.4\%)} & \textbf{86.73}\scalebox{0.6}{\color{gray}$\pm$0.43} \colorbox{lightgray}{(+1.2\%)} & 78.96\scalebox{0.6}{\color{gray}$\pm$0.96} \colorbox{lightgray}{(+1.3\%)} & +1.1\%\\
     \ours-\sage & 91.03\scalebox{0.6}{\color{gray}$\pm$0.44} \colorbox{lightgray}{(+1.4\%)} & \textbf{55.63}\scalebox{0.6}{\color{gray}$\pm$0.36} \colorbox{lightgray}{(+1.1\%)} & \textbf{98.82}\scalebox{0.6}{\color{gray}$\pm$0.37} \colorbox{lightgray}{(+0.7\%)}& 85.59\scalebox{0.6}{\color{gray}$\pm$0.38} \colorbox{lightgray}{(+0.9\%)} & 77.78\scalebox{0.6}{\color{gray}$\pm$0.88} \colorbox{lightgray}{(+0.5\%)} & +0.9\%\\
    \bottomrule
\end{tabular}
\label{tb:mainexp_heterophilic}
\end{table*}

\begin{table*}[t]
\centering
\setlength\fboxsep{1pt}
\setlength{\fboxrule}{1pt}
\caption{p-values obtained by comparing the performance of \ours and its corresponding backbone MPNN using one-tailed paired t-tests on heterophilic graphs. A p-value$<$0.05 ({\fcolorbox{c1}{white}{boxed}}) indicates that \ours's improvement over its backbone MPNN is statistically significant.}
\setlength{\tabcolsep}{1mm}
\begin{tabular}{cccccc}
    \toprule
     &\empire & \amazon & \minesweeper & \tolokers & \questions \\
    \midrule
     \ours-\gcn & {\fcolorbox{c1}{white}{2.3e-5}} & {\fcolorbox{c1}{white}{1.8e-3}} & {\fcolorbox{c1}{white}{5.0e-6}} & {\fcolorbox{c1}{white}{1.8e-2}} & {\fcolorbox{c1}{white}{4.4e-4}}\\
     \ours-\gat & {\fcolorbox{c1}{white}{3.5e-10}} & 8.8e-2 & {\fcolorbox{c1}{white}{1.3e-4}} & {\fcolorbox{c1}{white}{1.7e-5}} & {\fcolorbox{c1}{white}{1.2e-4}}\\
     \ours-\sage & {\fcolorbox{c1}{white}{7.1e-6}} & {\fcolorbox{c1}{white}{1.6e-4}} & {\fcolorbox{c1}{white}{2.0e-4}} & {\fcolorbox{c1}{white}{2.9e-3}} & {\fcolorbox{c1}{white}{9.4e-3}}\\
    \bottomrule
\end{tabular}
\label{tb:ttest}

\bigskip
\color{black}
\centering
\caption{Comparison of parameter count, runtime per epoch, and performance on \pascal and \tolokers datasets. \ours-GGCN denotes \ours-\gatedgcn. \ours achieves the best performance with a competitive model size and runtime.}
\setlength{\tabcolsep}{0.9mm}
\begin{tabular}{ccccccccccc}
    \toprule
    & \multicolumn{6}{c}{\pascal} & \multicolumn{4}{c}{\tolokers} \\
    \cmidrule(lr){2-7}\cmidrule(lr){8-11}
     & \geaet & \gps & \nba & \uniconv &\gatedgcn & \ours-GGCN &\iprmpnn & \nsquare &\gat & \ours-\gat \\
    \midrule
    \# params & 506,213 & 500,613 & 486,517 & 367,934 & 472,646 & 491,521 & 7,412,501 & 406,776 & 27,681 & 97,137 \\
     time (s/ep.) & 88.0 & 29.5 & 177.5 & 240.8 & 24.8 & 63.6 & 0.28 & 0.63 & 0.20 & 0.42 \\
     performance & 45.85\scalebox{0.6}{\color{gray} $\pm$0.87} & 44.40\scalebox{0.6}{\color{gray} $\pm$0.65} & 39.69\scalebox{0.6}{\color{gray} $\pm$0.27} & 40.05\scalebox{0.6}{\color{gray} $\pm$0.67} & 38.80\scalebox{0.6}{\color{gray} $\pm$0.40} & \textbf{48.28}\scalebox{0.6}{\color{gray} $\pm$0.54} & 82.00\scalebox{0.6}{\color{gray} $\pm$0.80} & 86.25\scalebox{0.6}{\color{gray} $\pm$0.41} & 85.66\scalebox{0.6}{\color{gray} $\pm$0.52} & \textbf{86.73}\scalebox{0.6}{\color{gray}$\pm$0.43} \\
    \bottomrule
\end{tabular}
\label{tb:efc}
\end{table*}

\paragraph{Performance on Heterophilic Graph Datasets} We also evaluate \ours using five heterophilic graph datasets from~\cite{hetegraph}: \empire, \amazon, \minesweeper, \tolokers, and \questions, each consisting of a single graph. The task is transductive node classification, i.e., test nodes are visible during training while their labels remain unknown. Following the standard evaluation protocol~\cite{hetegraph}, we report the mean and standard deviation across ten different splits. We adopt \gcn~\cite{gcn}, \gat~\cite{gat}, and \sage~\cite{gsage} as \ours's backbone MPNNs as they are commonly used in these datasets. We reproduced the backbones' results with additional hyperparameter tuning, including settings reported in~\cite{hetegraphexp}. Note that the results reported in~\cite{hetegraphexp} use only three of ten splits, making direct comparison infeasible.

Table~\ref{tb:mainexp_heterophilic} shows that \ours performs best on \amazon, \minesweeper, and \tolokers, and ranks second on \empire and third on \questions. When focusing on MPNN-based methods, \ours outperforms all baselines on all datasets except \questions. Also, \ours consistently improves the performance of all backbones: \gcn by 0.8\%, \gat by 1.1\%, and \sage by 0.9\% on average, yielding a total average improvement of 0.9\%. While the improvement is smaller than that on the LRGB, it can be attributed to the relatively strong performance of the backbones.

To verify that the performance improvement of \ours over its backbone MPNN is statistically significant, we conduct one-tailed paired t-tests comparing the performance of \ours with its corresponding backbone MPNN. Since \ours is evaluated with three backbones on each of the five datasets, we perform 15 tests in total. A p-value below 0.05 indicates that the improvement of \ours is statistically significant. The resulting p-values are reported in Table~\ref{tb:ttest}. \ours yields statistically significant improvement over the corresponding backbone in 14 out of 15 cases{\setlength\fboxsep{1pt} \setlength{\fboxrule}{1pt} \fcolorbox{c1}{white}{(boxed)}}. The only exception is \ours-\gat on \amazon, where the p-value slightly exceeds 0.05. These results confirm that \ours consistently improves the performance of its corresponding backbone MPNN on heterophilic graph datasets.
%We conducted one-tailed paired t-tests to evaluate whether \ours's improvement over its backbone MPNN is statistically significant. We observed that \ours yields statistically significant improvements (p-value$<$0.05) in 14 out of 15 cases, which are detailed in Appendix D.

\paragraph{Efficiency Comparison} Table~\ref{tb:efc} presents the number of parameters, runtime per epoch, and performance on \pascal and \tolokers. We report F1-macro(\%) for \pascal and AUC-ROC(\%) for \tolokers. On \pascal, \ours is slower than \gps and its backbone, \gatedgcn, but faster than all the other methods, while achieving the best performance. On \tolokers, \ours outperforms \iprmpnn and \nsquare with significantly fewer parameters. It runs slower than \iprmpnn and the backbone, \gat, but faster than \nsquare. \ours achieves state-of-the-art performance with a competitive model size and runtime.

\subsection{Ablation Studies and Qualitative Analysis}
\label{subsec:ablqual}

\begin{table}[t]
\centering
\caption{Ablation studies of \ours. We report AP for Pepfunc, F1-macro for Pascal, and AUC-ROC for Mine in \%.}
\setlength{\tabcolsep}{1mm}
\begin{tabular}{lllll}
    \toprule
     & Pepfunc ($\uparrow$) & Pascal ($\uparrow$) & Mine ($\uparrow$) \\
    \midrule
     (i) {$G^{(1)}=\cdots=G^{(L)}$} & 71.62\scalebox{0.6}{\color{gray}$\pm$0.41} & 47.52\scalebox{0.6}{\color{gray}$\pm$0.37} & 97.88\scalebox{0.6}{\color{gray}$\pm$0.43} \\
     (ii) all N-VN & 71.23\scalebox{0.6}{\color{gray}$\pm$0.61} & 40.21\scalebox{0.6}{\color{gray}$\pm$0.30} & 91.13\scalebox{0.6}{\color{gray}$\pm$2.83} \\
     (iii) {$\beta^{(l)}=0.0$} & 71.86\scalebox{0.6}{\color{gray}$\pm$0.44} & 47.42\scalebox{0.6}{\color{gray}$\pm$0.68} & 98.54\scalebox{0.6}{\color{gray}$\pm$0.35} \\
     (iv) {$\beta^{(l)}=1.0$} & 72.06\scalebox{0.6}{\color{gray}$\pm$0.46} & 47.19\scalebox{0.6}{\color{gray}$\pm$0.63} & 98.62\scalebox{0.6}{\color{gray}$\pm$0.32} \\
     (v) {$\alpha=0.0$} & 71.74\scalebox{0.6}{\color{gray}$\pm$0.45} & 40.02\scalebox{0.6}{\color{gray}$\pm$0.60} & 98.36\scalebox{0.6}{\color{gray}$\pm$0.40} \\
     (vi) {$\vx_z^{(l-1)}=\vq_z^{(l)}$} & 71.68\scalebox{0.6}{\color{gray}$\pm$0.49} & 45.48\scalebox{0.6}{\color{gray}$\pm$1.22} & 98.21\scalebox{0.6}{\color{gray}$\pm$0.38} \\
     (vii) all VN-VN & 72.02\scalebox{0.6}{\color{gray}$\pm$0.38} & 47.48\scalebox{0.6}{\color{gray}$\pm$0.52} & 98.47\scalebox{0.6}{\color{gray}$\pm$0.40} \\
    \midrule
     \ours & \textbf{72.38}\scalebox{0.6}{\color{gray}$\pm$0.52} & \textbf{48.28}\scalebox{0.6}{\color{gray}$\pm$0.54} & \textbf{98.82}\scalebox{0.6}{\color{gray}$\pm$0.37} \\
    \bottomrule
\end{tabular}
\label{tb:abl}
\end{table}

\paragraph{Ablation Studies} Table~\ref{tb:abl} presents the ablation studies of \ours on three datasets with different task types and evaluation settings: \pepfunc (Pepfunc) is a multi-graph dataset for inductive graph classification, \pascal (Pascal) is a multi-graph dataset for inductive node classification, and \minesweeper (Mine) is a single-graph dataset for transductive node classification.
\begin{enumerate}[label=\textbf{(\roman*)}]
\item introduces new VNs and connects them to nodes only at the first layer, keeping the graph fixed for subsequent layers. This consistently degrades performance on all three datasets, indicating that allowing VNs to be introduced at different layers is crucial for handling different tasks and settings.
\item connects all nodes to all selected VNs at each layer. This significantly degrades performance on all datasets, showing that fully connecting all nodes and VNs may result in suboptimal message passing. These results highlight the importance of selectively connecting nodes and VNs.
\item utilizes only the node-level preference scores $\tilde{s}_{vz|v}^{(l)}$, removing the VN-level preference scores $\tilde{s}_{vz|z}^{(l)}$ to compute the final edge scores $\tilde{s}_{vz}^{(l)}$ by setting $\beta^{(l)}=0$ in Section~\ref{subsec:w2vn}. This variation degrades performance, with a more pronounced drop on \pepfunc and \pascal than on \minesweeper. This is because, in the transductive single-graph setting, the connectivity pattern between nodes and VNs remains consistent between training and evaluation, making the preference scores less critical. In contrast, in the inductive multi-graph setting, the connections between nodes and VNs vary across graphs, and thus preference scores are essential for properly forming connections in new graphs.
\item utilizes only the VN-level preference scores $\tilde{s}_{vz|z}^{(l)}$, removing the node-level preference scores $\tilde{s}_{vz|v}^{(l)}$ to compute the final edge scores $\tilde{s}_{vz}^{(l)}$ by setting $\beta^{(l)}=1$ in Section~\ref{subsec:w2vn}. Similar to \textbf{(iii)}, this variation has a greater impact in the inductive setting than in the transductive setting.
\item uses only the original relevance scores for VN selections, node-VN connections, and VN-VN connections, removing the relative significance of these scores by setting $\alpha=0$ in Section~\ref{subsec:w2vn} and~\ref{subsec:vnvn}, thereby ignoring the $\texttt{logsoftmax}$ terms. In this variation, a node and a VN are connected solely based on their representations without accounting for the overall connectivity pattern. Ignoring the relative significance of the relevance scores in the VN selection and node-VN and VN-VN connection processes makes the model less adaptable to unseen graphs, which leads to a significant performance drop in the inductive setting, while its impact is less pronounced in the transductive setting, where the connectivity pattern remains fixed between training and evaluation.
\item solely learns VN representations instead of aggregating the representations of neighboring nodes by setting $\boldsymbol{\gamma}^{(l)}=\mathbf{1}_{d_{l-1}}$ in Section~\ref{subsec:vnrep}. Since VN representations are consistent between training and evaluation in the transductive setting, this component is more critical in the inductive setting.
\item fully connects all newly introduced VNs. Although this design is more critical for node-level tasks than graph-level tasks, it shows that the selective connectivity between VNs is needed for effective message passing.
\end{enumerate}
Overall, we note that \textbf{(i)} and \textbf{(ii)} impact all tasks and evaluation settings. On the other hand, \textbf{(iii)}-\textbf{(vi)} are more critical in the inductive multi-graph setting than in the transductive single-graph setting, while \textbf{(vii)} is more important for node-level tasks than for graph-level tasks. These results demonstrate that although the role of each component may vary depending on the tasks and evaluation settings, every component of \ours is essential.
%Table~\ref{tb:abl} presents ablation studies of \ours on \pepfunc(Pepfunc), \pascal(Pascal), and \minesweeper(Mine). These datasets cover both graph-level (\pepfunc) and node-level (\pascal, \minesweeper) tasks across both inductive (\pepfunc, \pascal) and transductive (\minesweeper) settings. Specifically, we consider the following variations: (i) using \ours only at the first layer and using the backbone MPNN for subsequent layers, representing the case where auxiliary connections are predetermined and fixed; (ii) connecting all nodes to the selected VNs ($\tilde{s}_{vz}^{(l)}\!=\!c\!\geq\!0$ in \S~\ref{subsec:w2vn}), (iii) using only node-level preference scores for node-VN connections ($\beta^{(l)}\!=\!0$ in \S~\ref{subsec:w2vn}); (iv) using only VN-level preference scores for node-VN connections ($\beta^{(l)}\!=\!1$ in \S~\ref{subsec:w2vn}); (v) independently establishing connections without accounting for the relative significance of relevance scores ($\alpha\!=\!0$ in \S~\ref{subsec:w2vn} and \S~\ref{subsec:vnvn}); (vi) learning VN representations instead of aggregating their neighboring nodes' representations ($\gamma^{(l)}=\mathbf{1}_{d_{l-1}}$ in \S~\ref{subsec:vnrep}); and (vii) interconnecting all selected VNs (\S~\ref{subsec:vnvn}). While the extent of performance degradation varies depending on the task type (graph-level vs. node-level) and evaluation setting (inductive vs. transductive), removing any component reduces performance, indicating that all components are essential to \ours. More explanations and analysis are provided in Appendix~\ref{app:abl}.

\begin{figure}[t]
\centering
\includegraphics[width=0.8\linewidth]{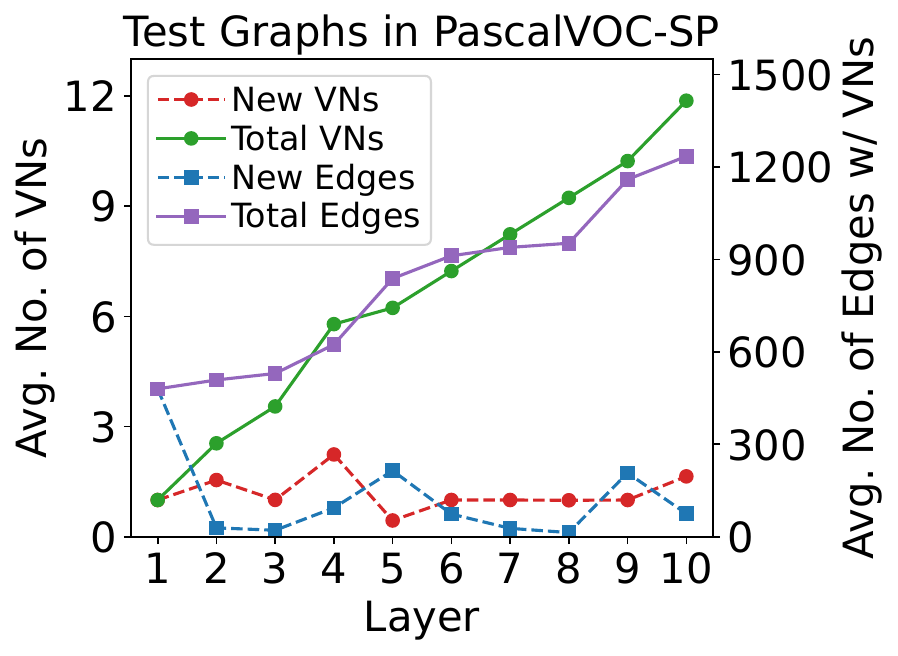}
\caption{Average number of introduced VNs and edges formed for each layer and the total number of VNs and introduced edges at each layer on the test graphs in \pascal.}
\label{fig:qual}
\end{figure}
\paragraph{Qualitative Analysis} 
 On test graphs in \pascal, we count the average number of VNs introduced and the average number of edges formed by \ours at each layer, shown in Figure~\ref{fig:qual}. Since \pascal consists of multiple graphs, we report the average number of new VNs and connections per layer across all graphs in the test set. We observe that both the number of new VNs and their connectivity patterns vary across layers. For instance, the fourth layer introduces an average of 2.24 VNs, whereas the fifth layer introduces only 0.44 VNs on average. While approximately one VN is introduced at the first, third, and sixth to ninth layers, the number of new edges on these layers varies substantially, ranging from 13.5 (layer 8) to 480.0 (layer 1) on average. The varying number of VNs and different connectivity patterns across layers indicate that \ours adaptively introduces VNs and flexibly forms connections.

%The first, third, and sixth to ninth layers each introduce approximately one VN per graph, but the number of new edges varies widely, ranging from 13.5 to 480.0.
%Also, while the first, third, and sixth to ninth layers introduce approximately one VN per graph, the number of new edges on these layers varies substantially, ranging from 13.5 to 480.0 on average per layer.

\begin{table}[t]
\centering
\caption{Runtime per epoch, peak memory usage, the average number of introduced VNs (Avg. VNs), and performance (F1-macro) of \ours with different values of $M$ on \pascal.}
\setlength{\tabcolsep}{1mm}
\begin{tabular}{ccccc}
    \toprule
     $M $& Time & Memory & Avg. VNs & F1-macro \\
    \midrule
     6 & 54 secs & 10.0 GB & 5.4 & 46.85\scalebox{0.6}{\color{gray}$\pm$0.92} \\
     12 & 60 secs & 12.3 GB & 9.5 & 47.37\scalebox{0.6}{\color{gray}$\pm$1.29} \\
     18 & 62 secs & 13.6 GB & 11.9 & 46.58\scalebox{0.6}{\color{gray}$\pm$0.41}\\
     24 & 64 secs & 14.0 GB & 11.9 & 48.28\scalebox{0.6}{\color{gray}$\pm$0.52}\\
     30 & 72 secs & 14.1 GB & 12.5 & 47.28\scalebox{0.6}{\color{gray}$\pm$0.88}\\
     36 & 75 secs & 14.5 GB & 14.3 & 47.32\scalebox{0.6}{\color{gray}$\pm$1.18}\\
    \bottomrule
\end{tabular}
\label{tb:M}
\end{table}

\paragraph{Varying $M$} Recall that $M$ is the upper bound of the number of VNs added by \ours, a hyperparameter chosen via grid search. To examine the impact of $M$ on both the computational cost and the performance, we report \ours's runtime, peak memory usage, and the average number of introduced VNs on \pascal for different values of $M$ in Table~\ref{tb:M}. We observe that the runtime, memory, and the average number of new VNs scale sublinearly with $M$. For instance, increasing $M$ from 6 to 36 raises the computation costs by only $1.5\times$. Such sublinear growth is attributed to \ours's score adjustment mechanisms proposed in \S~\ref{subsec:w2vn} and \S~\ref{subsec:vnvn}, which allow \ours to selectively add VNs and form edges rather than exhaustively add all candidate VNs and form all corresponding edges. Importantly, \ours outperforms the best baseline (45.85{\scalebox{0.6}{\color{gray}$\pm$0.87}}) across all tested values of $M$, demonstrating that it learns to leverage a useful subset of candidate VNs regardless of a specific choice of $M$.

%Table~\ref{tb:M} presents the runtime per epoch, peak memory usage, the average number of introduced VNs, and performance on \pascal for different values of the hyperparameter $M$. We observe that the runtime, memory usage, and the average number of introduced VNs increase sublinearly with $M$. For instance, employing 36 candidate VNs results in approximately 1.5 times higher runtime, peak memory usage, and the average number of introduced VNs compared to using 6 candidate VNs. This sublinear growth occurs because \ours selectively introduces candidate VNs through the score adjustment mechanisms proposed in \S~\ref{subsec:w2vn} and \S~\ref{subsec:vnvn}, rather than exhaustively introducing all candidate VNs. Consequently, \ours effectively regulates sparsity of the resulting graph structure for a given $M$.

\section{Conclusions and Future Work}
\label{sec:con}
We propose \ours, a novel MPNN framework that learns to determine \textit{when} and \textit{where} to introduce and connect VNs, allowing unconstrained and flexible connectivity between nodes and VNs. Since \ours makes no assumptions about the backbone architecture and relies solely on node representations to dynamically update message passing paths, it remains agnostic to the choice of backbone MPNNs. Experimental results on nine datasets, each evaluated with three different backbone MPNNs, show that \ours consistently improves the backbone MPNNs' performance and outperforms other state-of-the-art methods, underscoring the importance of adaptive message passing for MPNNs.

For future work, we aim to extend \ours into a graph coarsening framework. By interpreting the node-VN connections as cluster assignments, where VNs act as representative supernodes, we can leverage \ours to perform hierarchical graph pooling~\cite{megraph, vcr}. We expect this hierarchical abstraction to enhance the model's predictive performance across a diverse range of tasks by providing different levels of granularity for the given graph structure.

%As a future direction, we plan to develop a backbone MPNN explicitly tailored for \ours, operating under the assumption that auxiliary paths may be provided. Additionally, we aim to extend \ours to Graph Transformers~\cite{graphormer, gps, egt}, another family of GNNs capable of utilizing virtual nodes~\cite{geaet, exphormer}, by further improving the scalability of \ours. We also believe \ours can work with and complement asynchronous message passing methods such as \cognn~\cite{cognn} or \amp~\cite{amp}, which dynamically perform message passing on the given graph structure.

%As \ours relies on a backbone MPNN to update representations, its performance depends on its backbone. While \ours improves the backbone's performance by adding message passing paths, it adds computational overhead. Future work includes applying \ours to graph transformers, another family of GNNs that can utilize VNs, by further improving its scalability. Also, \ours can work with and complement asynchronous message passing methods such as \cognn~\cite{cognn} or \amp~\cite{amp}.

%%
%% The acknowledgments section is defined using the "acks" environment
%% (and NOT an unnumbered section). This ensures the proper
%% identification of the section in the article metadata, and the
%% consistent spelling of the heading.
\begin{acks}
This work was partly supported by the \grantsponsor{nrf}{National Research Foundation of Korea (NRF) grant funded by the Korean government (MSIT)}{https://www.nrf.re.kr/} (98\% from~\grantnum{nrf}{RS-2025-00559066}) and \grantsponsor{ai,lg}{Institute of Information \& communications Technology Planning \& Evaluation(IITP) grant funded by the Korea government(MSIT)}{} (1\% from~\grantnum{ai}{RS-2019-II190075}, Artificial Intelligence Graduate School Support Program(KAIST), 1\% from~\grantnum{lg}{RS-2025-25442149}, LG AI STAR Talent Development Program for Leading Large-Scale Generative AI Models in the Physical AI Domain).
\end{acks}

%%
%% The next two lines define the bibliography style to be used, and
%% the bibliography file.
\bibliographystyle{ACM-Reference-Format}
\bibliography{MAVN}

%%
%% If your work has an appendix, this is the place to put it.
\appendix

\section{Proof of Theorem 1}
\label{app:prf}
In Section~\ref{sec:propours}, we present Theorem~\ref{thm:flexours} to demonstrate that, for any set of additional message passing paths involving VNs, the parameters of a single-layer \ours can be configured to construct all paths in the set simultaneously, highlighting its flexibility in generating message passing paths. The proof of Theorem~\ref{thm:flexours} is based on the universal approximation theorem for MLPs~\citep{univ}.
\begin{thm:flexours}
Given a graph $G=(\sV,\sE)$ with $K$ sets of nodes $\sV_1, \sV_2, \cdots, \sV_K \subseteq \sV$, if node representations are uniquely distinguishable (i.e., $\vx_u=\vx_v\iff u=v, \forall u,v\in \sV$), there exists a parameter configuration of a single-layer \ours that introduces $K$ virtual nodes $z_1, z_2, \cdots, z_K$ into $G$, where each $z_i$ is connected to a node $v\in\sV$ if and only if $v\in\sV_i$.
\end{thm:flexours}
\begin{proof}
Let us set $M=K, d_\text{dot}=M$, and $\vk_{z_i}=\texttt{OneHot}(M, i)$, where $\texttt{OneHot}(M, i)$ denotes a vector of size $M$ with the $i$-th entry equal to 1 and all the other entries equal to 0. We use $\sX_G\subset\RR^{d}$ to denote the set of node representations, where $d$ is the dimension of node representations. Since the number of nodes is finite, $\sX_G$ is a compact set. The function $f$ maps a node representation $\vx_v$ to a binary vector of length $M$, where the $i$-th entry is 1 if node $v$ belongs to $\sV_i$ and 0 otherwise. Formally, $f$ is defined as follows:
\begin{equation}
f(\vx_v) = \sum_{i=1}^K \boldone[v\in\sV_i]\cdot\texttt{OneHot}(M, i),
\end{equation}
where $\boldone[\text{condition}]$ is the indicator function that returns 1 if the condition holds and 0 otherwise. Let $g:\RR^{d} \to \RR^{M}$ be a continuous function such that $\forall v\in \sV$, $g(\vx_v)=(2\epsilon+t\sqrt{d_\text{dot}})f(\vx_v) - \epsilon$, where $\epsilon>0$ is a positive real number and $t = \alpha\cdot\text{log}\big((M|\sV|-1)\exp(2\epsilon/\sqrt{d_\text{dot}})+1\big)+\text{log}(|\sV|)$.
Then, by universal approximation theorem for MLPs~\citep{univ}, there exists an MLP $\phi$ that approximates $g$ on the compact set $\sX_G$ with approximation error bounded by $\epsilon$. That is, for each $v\in \sV$, the following holds:
\begin{equation}
\begin{aligned}
|\phi(\vx_v) \cdot \texttt{OneHot}(M,i)-(2\epsilon+t\sqrt{d_\text{dot}}-\epsilon)| < \epsilon \quad& v \in \sV_i, \\ 
|\phi(\vx_v) \cdot \texttt{OneHot}(M,i)-(-\epsilon)| < \epsilon \quad& v \notin \sV_i.
\end{aligned}
\end{equation}
Let us use $\phi$ as the MLP in Eq.~\ref{eq:rel_score}. Since $\vk_{z_i}=\texttt{OneHot}(M,i)$, the relevance score $s_{vz_i}$ between a node $v$ and a VN $z_i$ lies in:
\begin{equation}
\begin{aligned}
t<&s_{vz_i} = \text{MLP}(\vx_v)\cdot \vk_{z_i}/\sqrt{d_\text{dot}}<2\epsilon/\sqrt{d_\text{dot}}+t &v \in \sV_i,\\
-2\epsilon<&s_{vz_i} = \text{MLP}(\vx_v)\cdot \vk_{z_i}/\sqrt{d_\text{dot}}<0 &v \notin \sV_i.
\end{aligned}
\end{equation}
Since $s_{vz_i} < 0$ when $v\notin\sV_i$, and \texttt{logsoftmax} values are non-positive by definition, it is clear that $z_i$ will never be connected to $v$ if $v\notin\sV_i$. For the same reason, $z_i$ will not be selected if $\sV_i=\emptyset$.

Therefore, to complete the proof, it remains to show that 1) $z_i$ is selected if $\sV_i\neq\emptyset$, and 2) if $v\in\sV_i$, then $z_i$ is connected to $v$. 

When $v\in\sV_i$, the lower bound of $\bar{s}_{vz_i}$ is computed by:
\begin{align*}
    \bar{s}_{vz_i}&=s_{vz_i}+\alpha\cdot\text{logsoftmax}(\sS)[s_{vz_i}]\\
    &=s_{vz_i}+\alpha\text{log}\left(\frac{\text{exp}(s_{vz_i})}{\sum_{j=1}^M\sum_{u\in\sV}\text{exp}(s_{uz_j})}\right)\\
    &>s_{vz_i}+\alpha\text{log}\left(\frac{\text{exp}(t)}{\sum_{j=1}^M\sum_{u\in\sV}\text{exp}(s_{uz_j})-\text{exp}(s_{vz_i})+\text{exp}(t)}\right)\\
    &>t+\alpha t-\alpha\text{log}\left(\sum_{j=1}^M\sum_{u\in\sV}\text{exp}(s_{uz_j})-\text{exp}(s_{vz_i})+\text{exp}(t)\right)\\
    &>(1+\alpha)t-\alpha\text{log}\left((M|\sV|-1)\text{exp}\left(t+2\epsilon/\sqrt{d_\text{dot}}\right)+\text{exp}(t)\right)\\
    &=(1+\alpha)t-\alpha\left(t+\text{log}\left((M|\sV|-1)\text{exp}\left(2\epsilon/\sqrt{d_\text{dot}}\right)+1\right)\right)\\
    &=(1+\alpha)t-\alpha t-(t-\text{log}(|\sV|)) =\text{log}(|\sV|),
\end{align*}
where $\sS=\{s_{v'z'}|v'\in\sV, z'\in\{z_1, z_2, \cdots, z_K\}\}$. The first inequality holds because for any positive real numbers $a,b,c$, if $b>c$, then $b/(a+b)>c/(a+c)$.

Also, the score $s_{z_i}$ of VN $z_i$ is bounded by:
\begin{equation}
\begin{aligned}
    s_{z_i}&=\text{log}(\sum_{u\in\sV}\text{exp}(\bar{s}_{uz_i})/|\sV|) \geq \text{log}(\text{exp}(\bar{s}_{vz_i})/|\sV|) \\
    &> \text{log}(|\sV|/|\sV|)=\text{log}(1)=0.
\end{aligned}
\end{equation}
$s_{z_i}>0$ implies that $\texttt{sigmoid}(s_{z_i})>0.5$, and thus, the VN $z_i$ is selected. Moreover, since $\bar{s}_{vz_i} > \text{log}(|\sV|) > 0$, a connection between $z_i$ and $v$ is formed. Note that $\bar{s}_{vz_i} > 0$ is a sufficient condition for this connection, since $\bar{s}_{vz_i} \leq \tilde{s}_{vz_i|z_i}$ and $ \bar{s}_{vz_i} \leq \tilde{s}_{vz_i|v}$ when $\sS_{:z_i}, \sS_{v:}\subseteq \sS$.

%given that $\sS_{:z_i}, \sS_{v:}$ are subsets of $\sS$.

Therefore, there exists a parameter configuration of a single-layer \ours that establishes a connection between a node $v$ and a VN $z_i$ if and only if $v\in\sV_i$.
\end{proof}

\section{Details of \ours}
\label{app:detours}
Our codes are provided in \url{https://github.com/bdi-lab/MAVN}.

\subsection{Implementation Details of \textbf{\ours}}
\label{app:impl}
In our implementation, we apply separate normalizations to the backbone MPNN, the input of $\text{MLP}^{(l)}$, and $\vk_z^{(l)}$. For graph-level tasks, we additionally introduce a global VN into $G^{(0)}$ that serves as the representation node for the entire graph. 

For backbone MPNNs used on the LRGB datasets~\citep{lrgb}, we follow the implementations from~\cite{lrgbexp}. For the heterophilic graph datasets~\citep{hetegraph}, we employ two different backbone implementations: one from~\cite{hetegraph} and another from~\cite{hetegraphexp}. The choice of implementation is based on validation performance. For more details about the implementations, please refer to our code.

\subsection{Experimental Details of \textbf{\ours}}
\label{app:exp}
All experiments are conducted using PyTorch 2.0.1 with cudatoolkit 11.7 and python 3.9.19 on Ubuntu 18.04. Our computing infrastructure has 512GB of memory, with two Intel(R) Xeon(R) Gold 6330 CPU @ 2.00GHz. We used LRGB~\cite{lrgb} and heterophilic graph datasets~\cite{hetegraph}, as they are well-established benchmarks in the literature. Hyperparameters are provided along with the code.

\section{Dataset Details}
\label{app:data}

Table~\ref{tb:stat:syn} presents the statistic for two multi-graph synthetic datasets, \synunder~\citep{prmpnn} (\treelc) and \synsquash~\citep{ovsq} (\treenb). “dep." is the tree depth. In both datasets, the task is to predict the label of the root node. These datasets are designed to evaluate a model's ability to mitigate under-reaching and over-squashing, respectively. Edge features are not provided in these datasets.

Table~\ref{tb:lrgb} shows the statistic for the four LRGB datasets \citep{lrgb}. These are multi-graph datasets with disjoint training, validation, and test graphs; the tasks are performed in an inductive setting. \pepfunc and \pepstruct are for graph classification and graph regression tasks, respectively, on peptides molecular graphs. \pascal and \coco represent images as graphs, where each node corresponds to a superpixel, and the task is node classification. 

\begin{table}
\centering
\caption{Dataset statistic for two tree synthetic datasets with depths ranging from 4 to 6. “Acc." denotes the accuracy.}
\label{tb:stat:syn}
\setlength{\tabcolsep}{1mm}
\begin{tabular}{cccccccc}
    \toprule
     & dep. &\#graphs & Avg. $|\sV|$ & Avg. $|\sE|$ & $d'$ & $d''$ & Metric\\
    \midrule
    \multirow{3}{*}{\treelc} & 4 & 14,000 & 31 & 30 & 2 & N/A &Acc.\\
    & 5 & 11,997 & 63 & 62 & 2 & N/A & Acc.\\
    & 6 & 12,597 & 127 & 126 & 2 & N/A &Acc.\\
    \midrule
    \multirow{3}{*}{\treenb} & 4 & 16,000 & 31 & 30 & 2 & N/A &Acc.\\
    & 5 & 32,000 & 63 & 62 & 2 & N/A &Acc.\\
    & 6 & 32,000 & 127 & 126 & 2 & N/A &Acc.\\
    \bottomrule
\end{tabular}
\end{table}

\begin{table}
\centering
\caption{Dataset statistic for the four LRGB datasets.}
\setlength{\tabcolsep}{1mm}
\begin{tabular}{ccccccc}
    \toprule
     &\#graphs & Avg. $|\sV|$ & Avg. $|\sE|$ & $d'$ & $d''$ & Metric\\
    \midrule
    \pepfunc & 15,535 & 150.9 & 153.7 & 9 & 3 & AP\\
    \pepstruct & 15,535 & 150.9 & 153.7 & 9 & 3 & MAE\\
    \pascal & 11,355 & 479.4 & 1355.2 & 14 & 2 & F1-macro\\
    \coco & 123,286 & 476.9 & 1346.8 & 14 & 2 & F1-macro\\
    \bottomrule
\end{tabular}
\label{tb:lrgb}
\end{table}

\begin{table}
\centering
\caption{Statistic for the five Heterophilic Graph Datasets.}
\setlength{\tabcolsep}{1mm}
\begin{tabular}{ccccccc}
    \toprule
     & $|\sV|$ & $|\sE|$ & $d'$ & $d''$ & Metric\\
    \midrule
    \empire & 22,662 & 32,927 & 300 & N/A & Accuracy\\
    \amazon & 24,492 & 93,050 & 300 & N/A & Accuracy\\
    \minesweeper & 10,000 & 39,402 & 7 & N/A & AUCROC\\
    \tolokers & 11,758 & 519,000 & 10 & N/A & AUCROC\\
    \questions & 48,921 & 153,540 & 301 & N/A & AUCROC\\
    \bottomrule
\end{tabular}
\label{tb:hete}
\end{table}

Table~\ref{tb:hete} shows the dataset statistic for the five Heterophilic Graph Datasets~\citep{hetegraph}. Each dataset consists of a single graph, and the task is transductive node classification on the given graph. \empire is a word co-occurrence graph derived from the English Wikipedia article on the Roman Empire. \amazon represents a product co-purchasing network from Amazon. \minesweeper is a synthetic grid graph simulating the Minesweeper game. \tolokers is a user interaction graph from the Toloka crowdsourcing platform, and \questions is an interaction graph from the question-answering website Yandex Q. Edge features are not provided in these datasets.

\section{Additional Experiments on \ours}
\label{app:prop}

\begin{figure}
    \centering
    \includegraphics[width=0.9\linewidth]{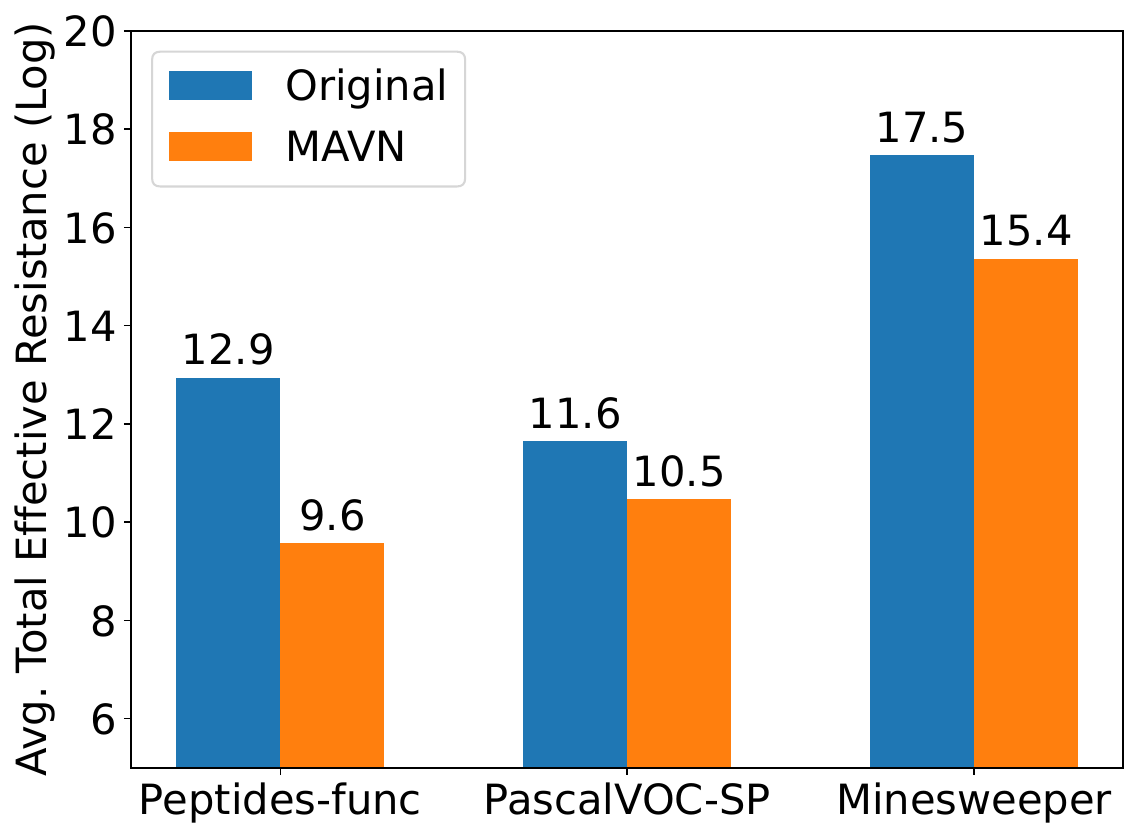}
    \caption{Comparison of log-transformed average total effective resistance between the original graph and the graph produced by \ours. The consistent reduction in total effective resistance suggests that \ours alleviates over-squashing.}
    \label{fig:tot_er}
\end{figure}

\begin{figure}
\centering
\includegraphics[width=0.9\linewidth]{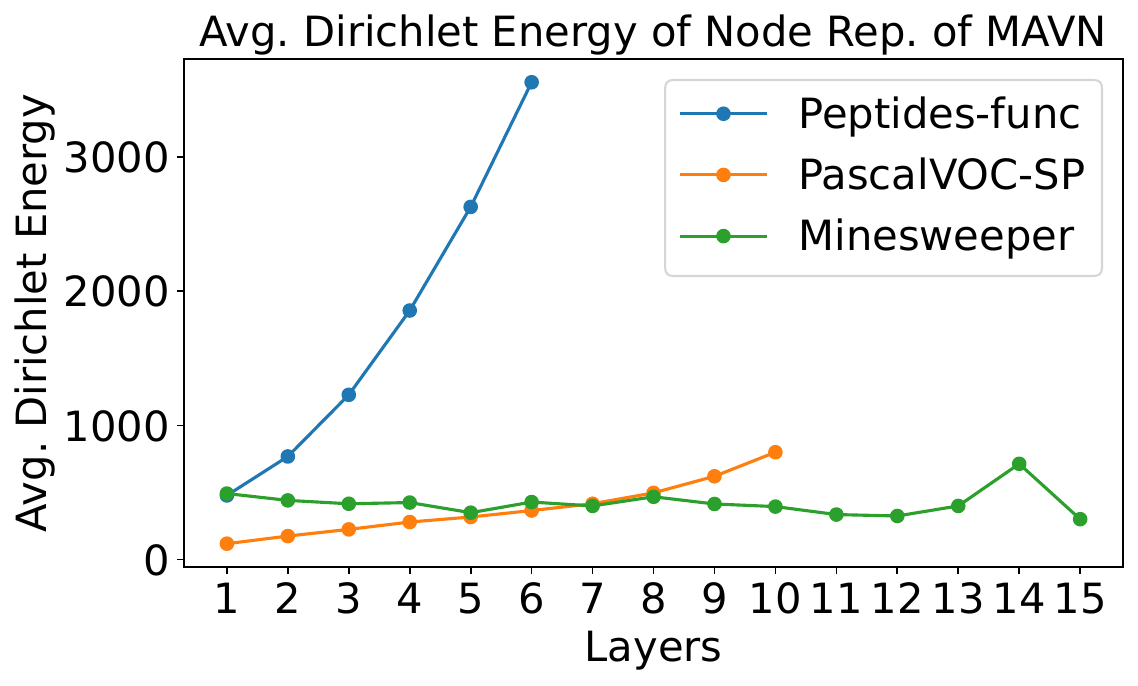}
\caption{Average Dirichlet energy of node representations computed by \ours at each layer. Dirichlet energy does not decay exponentially with increasing depth, indicating that \ours does not lead to over-smoothing.}
\label{fig:de}
\end{figure}

\paragraph{Over-squashing} Total effective resistance, also known as effective graph resistance~\citep{egr}, is defined as the sum of effective resistances between all pairs of nodes in a graph. It serves as an upper bound on the extent of over-squashing~\citep{ovsqer}. Thus, a lower total effective resistance implies a reduced degree of over-squashing. Figure~\ref{fig:tot_er} compares the log-transformed total effective resistance between the original graph $G^{(0)}$ and the final graph $G^{(L)}$ of \ours on three real-world datasets: \pepfunc, \pascal, and \minesweeper. For the multi-graph datasets, \pepfunc and \pascal, the total effective resistance values are averaged across the test graphs. For \minesweeper, a single-graph dataset, the total effective resistance values are averaged over the final graphs produced by \ours trained on each of ten splits. Across all datasets, \ours consistently reduces the total effective resistance, indicating that \ours effectively alleviates over-squashing in real-world datasets.

\paragraph{Over-smoothing} Over-smoothing is a phenomenon in which node representations become overly similar~\citep{ovsm}. It can be characterized as a layer-wise exponential decay of Dirichlet energy~\citep{ovsmsurv}, where Dirichlet energy quantifies differences between node representations. Thus, if over-smoothing occurs, an exponential decrease in Dirichlet energy should be observed with increasing layers. Figure~\ref{fig:de} illustrates the Dirichlet energy of node representations computed by \ours at each layer on \pepfunc, \pascal, and \minesweeper, using the definition of Dirichlet energy from Eq.~2 in~\cite{ovsmsurv}. For the multi-graph datasets, \pepfunc and \pascal, we report the Dirichlet energy averaged over the test graphs. For \minesweeper, a single-graph dataset, we report the average Dirichlet energy of node representations of \ours trained on ten different splits. On all datasets, over-smoothing does not occur, as the Dirichlet energy does not exhibit a consistent decrease. Specifically, on \pepfunc and \pascal, the Dirichlet energy consistently increases with depth. On \minesweeper, the trend is non-monotonic, with both increases and decreases observed at different layers. These results demonstrate that \ours does not introduce over-smoothing in real-world datasets.

\end{document}